\documentclass{article} 
\usepackage{iclr2022_conference,times}

\usepackage{amsmath,amsfonts,bm}

\def\eqref#1{equation~\ref{#1}}

\def\1{\bm{1}}

\DeclareMathAlphabet{\mathsfit}{\encodingdefault}{\sfdefault}{m}{sl}
\SetMathAlphabet{\mathsfit}{bold}{\encodingdefault}{\sfdefault}{bx}{n}

\newcommand{\E}{\mathbb{E}}

\DeclareMathOperator*{\argmin}{arg\,min}

\usepackage{hyperref}       
\usepackage{url}            
\usepackage{booktabs}       
\usepackage{amsfonts}       
\usepackage{nicefrac}       
\usepackage{xcolor}         
\usepackage{lipsum}
\usepackage{subfiles}
\usepackage{multirow}
\usepackage{booktabs}
\usepackage{subcaption}
\usepackage{graphicx}
\usepackage{wrapfig}
\usepackage{amssymb}
\usepackage{amsbsy}
\usepackage{xspace}
\usepackage{amssymb,amsmath,amsthm}
\usepackage{adjustbox}
\usepackage{threeparttable}
\usepackage[mathscr]{euscript}
\usepackage{comment}
\usepackage{enumitem}
\setlist[itemize]{leftmargin=10pt}

\usepackage{nicefrac}
\usepackage{fdsymbol}

\newcommand{\bluetext}[1]{{#1}}
\newcommand{\newbluetext}[1]{{#1}}

\newcommand{\ours}{\texttt{PORT}\xspace}

\newcommand{\wrn}[1]{\textsc{Wrn-#1}\xspace}
\newcommand{\vikash}[1]{{\color{red}}}

\newcommand{\dtrain}{D\xspace}
\newcommand{\dsyn}{\tilde{D}\xspace}
\newcommand{\approxdist}{proxy\xspace}

\newcommand{\orgdist}{original\xspace}
\newcommand{\realdist}{real data\xspace}

\newcommand{\aurpc}{ARC\xspace}
\newcommand{\aurpctt}{\texttt{ARC}\xspace}
\newcommand{\aurpcsf}{\mathsf{ARC}\xspace}

\newcommand{\synthetic}{synthetic\xspace}
\newcommand{\cwdfull}{conditional Wasserstein distance\xspace}

\newcommand{\oneNN}{1-NN}

\newcommand{\styleC}{StyleGAN\xspace}
\newcommand{\ddpm}{DDPM\xspace}
\newcommand{\cwd}{\mathsf{cwd}}
\newcommand{\Racc}{\mathsf{Racc}}

\newcommand{\fhat}[2]{\ifthenelse{\equal{#2}{}}{\hat{f}(#1)}{\ifthenelse{\equal{#2}{0}}{\hat{f}(\emptyset)}{\hat{f}(#1_{\leq #2})}}}
\newcommand{\ftild}[2]{\ifthenelse{\equal{#2}{}}{\tilde{f}(#1)}{\ifthenelse{\equal{#2}{0}}{\tilde{f}(\emptyset)}{\tilde{f}(#1_{\leq #2})}}}
\newcommand{\ftildstar}[2]{\ifthenelse{\equal{#2}{}}{\tilde{f^*}(#1)}{\ifthenelse{\equal{#2}{0}}{\tilde{f^*}(\emptyset)}{\tilde{f^*}(#1_{\leq #2})}}}
\newcommand{\ghat}[2]{\ifthenelse{\equal{#2}{}}{\hat{g}(#1)}{\ifthenelse{\equal{#2}{0}}{\hat{g}(\emptyset)}{\hat{g}(#1_{\leq #2})}}}
\newcommand{\mfix}[2]{\ifthenelse{\equal{#2}{}}{m(#1)}{\ifthenelse{\equal{#2}{0}}{m(\emptyset)}{m(#1_{\leq #2})}}}

\newcommand{\fapp}[2]{\ifthenelse{\equal{#2}{}}{\tilde{f}(#1)}{\ifthenelse{\equal{#2}{0}}{\tilde{f}(\emptyset)}{\tilde{f}(#1_{\leq #2})}}}

\newcommand{\Rob}{\mathsf{Rob}}

\newcommand{\aSF}{\mathsf{a}}
\newcommand{\gSF}{\mathsf{g}}
\newcommand{\hSF}{\mathsf{h}}

\newcommand{\avr}[2]{\ifthenelse{\equal{#2}{}}{\aSF({#1})}{\ifthenelse{\equal{#2}{0}}{\aSF(\emptyset)}{\aSF({#1}_{\leq #2})}}}

\newcommand{\avrMax}[2]{\ifthenelse{\equal{#2}{}}{\aSF^*({#1})}{\ifthenelse{\equal{#2}{0}}{\aSF^*(\emptyset)}{\aSF^*({#1}_{\leq #2})}}}

\newcommand{\avrApp}[2]{\ifthenelse{\equal{#2}{}}{\tilde{\aSF}({#1})}{\ifthenelse{\equal{#2}{0}}{\tilde{\aSF}(\emptyset)}{\tilde{\aSF}({#1}_{\leq #2})}}}

\newcommand{\avrAppMax}[2]{\ifthenelse{\equal{#2}{}}{\tilde{\aSF}^*({#1})}{\ifthenelse{\equal{#2}{0}}{\tilde{\aSF}^*(\emptyset)}{\tilde{\aSF}^*({#1}_{\leq #2})}}}

\newcommand{\ArgMax}[2]{\ifthenelse{\equal{#2}{}}{\hSF({#1})}{\ifthenelse{\equal{#2}{0}}{\hSF(\emptyset)}{\hSF({#1}_{\leq #2})}}}

\newcommand{\AppArgMax}[2]{\ifthenelse{\equal{#2}{}}{\tilde{\hSF}({#1})}{\ifthenelse{\equal{#2}{0}}{\tilde{\hSF}(\emptyset)}{\tilde{\hSF}({#1}_{\leq #2})}}}

\newcommand{\gain}[2]{\ifthenelse{\equal{#2}{}}{\gSF(#1)}{\gSF(#1_{\leq #2})}}
\newcommand{\gainMax}[2]{\ifthenelse{\equal{#2}{}}{\gSF^*(#1)}{\gSF^*(#1_{\leq #2})}}
\newcommand{\gainApp}[2]{\ifthenelse{\equal{#2}{}}{\tilde{\gSF}(#1)}{\tilde{\gSF}(#1_{\leq #2})}}
\newcommand{\gainAppMax}[2]{\ifthenelse{\equal{#2}{}}{\tilde{\gSF}^*(#1)}{\tilde{\gSF}^*(#1_{\leq #2})}}

\newcommand{\pr}[2][]{\Pr_{\ifthenelse{\isempty{#1}}{}{{#1}}}\left[{#2}\right]}

\newcommand{\remove}[1]{}

\newcommand{\cJ}{{\mathcal J}}

\newcommand{\cS}{{\mathcal S}}

\newcommand{\cX}{{\mathcal X}}
\newcommand{\cY}{{\mathcal Y}}
\newcommand{\cZ}{{\mathcal Z}}

\newcommand{\Exp}{\operatorname*{\mathbb{E}}}
\newcommand{\Ex}{\Exp}

\newtheorem{theorem}{Theorem}

\newtheorem{lemma}[theorem]{Lemma}

\newtheorem{definition}{Definition}

\newcommand{\sdotfill}{\textcolor[rgb]{0.8,0.8,0.8}{\dotfill}} 

\makeatletter
\def\th@protocol{%
    \normalfont 
    \setbeamercolor{block title example}{bg=orange,fg=white}
    \setbeamercolor{block body example}{bg=orange!20,fg=black}
    \def\inserttheoremblockenv{exampleblock}
  }
\makeatother

\newtheorem{proto}[theorem]{Protocol}
\newtheorem{protoc}[theorem]{Protocol}

\newcommand{\namedref}[2]{#1~\ref{#2}}

\newcommand{\torestate}[3]{%
\expandafter \def \csname BBRESTATE #2 \endcsname{#3}
\theoremstyle{plain}
\newtheorem{BBRESTATETHMNUM#2}[theorem]{#1}
\begin{BBRESTATETHMNUM#2}\label{#2}\csname BBRESTATE #2 \endcsname   \end{BBRESTATETHMNUM#2}
\newtheorem*{BBRESTATETHMNONNUM#2}{\namedref{#1}{#2}}
}

\newcommand{\restate}[1]{\begin{BBRESTATETHMNONNUM#1}[Restated] \csname BBRESTATE #1 \endcsname
\end{BBRESTATETHMNONNUM#1}}

\newcommand\blfootnote[1]{%
  \begingroup
  \renewcommand\thefootnote{}\footnote{#1}%
  \addtocounter{footnote}{-1}%
  \endgroup
}

\title{Robust Learning Meets Generative Models: Can Proxy Distributions Improve Adversarial Robustness?}

\author{
  Vikash Sehwag$^{\spadesuit\dagger}$, Saeed Mahloujifar$^\spadesuit$, Tinashe Handina$^\vardiamondsuit$, Sihui Dai$^\spadesuit$, Chong Xiang$^\spadesuit$
  \\
  \textbf{Mung Chiang}$^\clubsuit$, \textbf{Prateek Mittal}$^\spadesuit$ \vspace{10pt}\\
  $^\spadesuit$Princeton University, $^\vardiamondsuit$Caltech, $^\clubsuit$Purdue University
}

\newcommand{\newtext}[1]{#1}

\iclrfinalcopy 
\begin{document}

\maketitle

\begin{abstract}
While additional training data improves the robustness of deep neural networks against adversarial examples, it presents the challenge of curating a large number of specific real-world samples. We circumvent this challenge by using additional data from proxy distributions learned by advanced  generative models. We first seek to formally understand the transfer of robustness from classifiers trained on proxy distributions to the real data distribution. We prove that the difference between the robustness of a classifier on the two distributions is upper bounded by the conditional Wasserstein distance between them. Next we use proxy distributions to significantly improve the performance of adversarial training on \textit{five} different datasets. For example, we improve robust accuracy by up to $7.5$\% and $6.7$\% in $\ell_{\infty}$ and $\ell_2$ threat model over baselines that are not using proxy distributions on the CIFAR-10 dataset. We also improve certified robust accuracy by $7.6$\% on the CIFAR-10 dataset.
We further demonstrate that different generative models bring a disparate improvement in the performance in robust training. We propose a robust discrimination approach to characterize the impact of individual generative models and further provide a deeper understanding of why current state-of-the-art in diffusion-based generative models are a better choice for proxy distribution than generative adversarial networks.

\end{abstract}

\section{Introduction}
\blfootnote{$^\dagger$Corresponding author: \texttt{vvikash@princeton.edu}}
\blfootnote{Our code is available at \url{https://github.com/inspire-group/proxy-distributions}.}
Deep neural networks are powerful tools but their success depends strongly on the amount of training data~\citep{sun2017revisitData, mahajan2018BillionWSL}. 
Recent works show 
that for improving robust training against adversarial examples~\citep{biggio2013evasion, Szegedy2013IntrigueAdv, biggio2018wildPatterns}, additional training data is 
helpful~\citep{carmon2019unlabeled, schmidt2018AdvMoreData, alayrac2019unsupadv, deng2020AdvExtraOutDomain}.
However, curating more real-world data from the actual data distribution is usually 
challenging and costly~\citep{recht2019ImageNetv2}. 
To circumvent this challenge, we ask: can robust training be enhanced using a proxy distribution, i.e., an approximation of the \realdist distribution? 
In particular, can additional samples from a proxy distribution, that is perhaps cheaper to sample from, improve robustness. If so, can generative models that are trained on limited training images 
in small scale datasets~\citep{lecun2010mnist, krizhevsky2014cifar}, act as such a proxy distribution?\footnote{Proxy distributions may not necessarily be modeled by generative models. When a proxy distribution is the output of a generative model, we call it \emph{synthetic distribution} and refer to data sampled from it as \emph{synthetic data}.}

When training on synthetic samples from generative models, a natural question is whether robustness on synthetic data will also transfer to real world data. Even if it does, can we determine the features of synthetic data that enable this \emph{synthetic-to-real} robustness transfer and optimize our selection of generative models based on these features? Finally, can we also optimize the selection of individual synthetic samples to maximize the robustness transfer? 

\emph{Q.1} \emph{When does robustness transfer from proxy distribution to real data distribution?} This question is fundamental to develop a better understanding of whether a proxy distribution will help. \bluetext{For a classifier trained on only synthetic samples, we argue that in addition to empirical and generalization error, \emph{distribution shift} penalty also determines its robustness on the real data distribution.} We prove that this penalty is upper bounded by the \cwdfull between the \approxdist distribution and \realdist  distribution. Thus robustness will transfer from proxy distributions that are in close proximity to \realdist distribution with respect to \cwdfull. 

\emph{Q.2} \emph{\newtext{How effective are proxy distributions in boosting adversarial robustness on real-world dataset?}}
\newtext{Our experimental results on \textit{five} datasets demonstrate that the use of samples from \textit{PrOxy distributions in Robust Training} (PORT) can significantly improve robustness. In particular, PORT achieves up to to $7.5$\% improvement in adversarial robustness over existing state-of-the-art~\citep{croce2020robustbench}. Its improvement is consistent across different threat models ($\ell_{\infty}$ or $\ell_2$), network architectures, datasets, and robustness criteria (empirical or certified robustness).  We also uncover that synthetic images from diffusion-based generative models are most helpful in improving robustness on real datasets. We further investigate the use of proxy distributions in robust training. In particular, we investigate why current state-of-the-art in diffusion-based models are significantly more helpful than their counterparts, such as generative adversarial networks (GANs). 
}

\newtext{
\emph{Q.3} \emph{Can we develop a metric to characterize which proxy distribution will be most helpful in robust training?} 
Our theory motivates the design of a measure of proximity between two distributions that incorporates the geometry of the distributions and can empirically predict the transfer of robustness.}
We propose a robust learning based approach where we use the success of a discriminator in distinguishing adversarially perturbed samples of synthetic and real data as a measure of proximity. Discriminating between synthetic and real data is a common practice~\citep{goodfellow2014GAN, gui2020ganreview},
\bluetext{however, we find that considering adversarial perturbations on synthetic and real data samples is the key to making the discriminator an effective measure for this task.} 
We demonstrate that the rate of decrease in discriminator success with an increasing size of perturbations can effectively measure proximity and it accurately predicts the relative transfer of robustness from different generative models. 

We also leverage our robust discriminators to identify most helpful synthetic samples. We use the proximity of each synthetic sample to \realdist distribution, referred to as synthetic score, as a metric to judge their importance. This score can be computed using output probability from a discriminator that is robustly trained to distinguish between adversarially perturbed samples from \approxdist and \realdist distribution. \bluetext{We demonstrate that selecting synthetic images based on their synthetic scores can further improve performance.}  


\noindent \textbf{Contributions.}
We make the following key contributions.
\begin{itemize}[noitemsep, topsep=0pt]
    \item We provide an analytical bound on the transfer of robustness from proxy distributions to \realdist distribution using the notion of conditional Wasserstein distance and further validate it experimentally.

    \item Overall, using additional synthetic images, we improve both clean and robust accuracy on \textit{five} different datasets. In particular, we improve robust accuracy by up to $7.5$\% and $6.7$\% in $\ell_{\infty}$ and $\ell_2$ threat models, respectively, and certified robust accuracy ($\ell_2$) by $7.6$\% on the CIFAR-10 dataset.
    
    \item When selecting a proxy distribution, we show that existing metrics, such as FID, fail to determine the synthetic-to-real transfer of robustness. We propose a new metric (\emph{\aurpc}) based on the distinguishability of adversarially perturbed synthetic and real data, that accurately determines the performance transfer.  
     
    \item We also develop a metric, named synthetic score, to determine the importance of each synthetic sample in synthetic-to-real robustness transfer. We demonstrate that choosing samples with lower synthetic scores provide better results than randomly selected samples.
\end{itemize}

\section{Integrating proxy distributions in robust training} \label{sec: methods}
In this section we propose to use samples from proxy distributions to improve robustness. We first provide analytical bounds on the transfer of adversarial robustness from \approxdist to \realdist distribution. Next, using robust discriminators we provide a metric based on our analytical bound, which can accurately determine the relative ranking of different proxy distributions in terms of robustness transfer. Our metric can be calculated empirically (using samples from both distributions) and does not require the knowledge of the \approxdist or \realdist distribution. Finally, we present our robust training formulation (PORT) which uses synthetic samples generated by the generative model, together with real samples.


\noindent \textbf{Notation.} We represent the input space by $\mathcal{X}$ and corresponding label space as $\mathcal{Y}$. Data is sampled from a joint distribution $D$ that is supported on $\mathcal{X}\times \mathcal{Y}$. For a label $y$, we use $D\mid y$ to denote the conditional distribution of class $y$. 
We denote the proxy distribution as $\dsyn$.  We denote the neural network for classification by $f: \cX \rightarrow \cZ$, parameterized by $\theta$, which maps input images to output probability vectors ($z$). We use $h$ to refer to the classification functions that output labels. For a set $\cS$ sampled from a distribution $D$, we use $\hat{\cS}$ to denote the empirical distribution with respect to set $\cS$. We use $\mathcal{S}\sim D$ to denote the sampling of a dataset from a distribution $D$. We use $(x,y)\gets D$ to denote the sampling of a single point from $D$.



\subsection{Understanding transfer of adversarial robustness between data distributions} \label{sec: theory_robstness_bound}
\bluetext{
Since our goal is to use samples from proxy distribution to improve robustness on \realdist distribution, we first study the transfer of adversarial robustness between two data distributions.
}


\begin{definition}[Average Robustness]
We define average robustness for a classifier $h$  on a distribution $D$ according to a distance metric $d$ as follows: $\Rob_d(h,D)=\Ex_{(x,y)\gets D}[\inf_{h(x')\neq y } d(x',x)].$
where classifier $h: \cX \rightarrow \cY$ predicts class label of an input sample
\footnote{This notion of robustness is also used in ~\citet{gilmer2018adversarial,diochnos2018adversarial} and \citet{mahloujifar2019curse}.}.\end{definition}
This definition refers to the expected distance to the closest adversarial example for each~sample. 


\noindent \textbf{Formalizing transfer of robustness from \approxdist to \realdist  distribution.} In robust learning from a proxy distribution, we are interested in bounding the average robustness of the classifier obtained by a learning algorithm ($L$), on distribution $D$, when the training set is a set $\cS$ of $n$ labeled examples sampled from a proxy distribution $\tilde{D}$. In particular we want to provide a lower bound on the \bluetext{transferred average robustness} i.e, $\Ex_{\substack{\cS \sim \tilde{D}\\ h\gets L(S)}}[\Rob_d(h,D)].$

In order to understand this quantity better, suppose $h$ is a classifier trained on a set $\cS$ that is sampled from $\tilde{D}$, using algorithm $L$. We decompose $\Rob_d(h,D)$ to three quantities as follows:
\begin{align*}
\Rob_d(h,D) = \big(\Rob_d(h,D)-\Rob_d(h,\tilde{D})\big) + \big(\Rob_d(h,\tilde{D}) - \Rob_d(h,\hat{\cS})\big) + \Rob_d(h,\hat{\cS}).
\end{align*}
Using this decomposition, by linearity of expectation and triangle inequality we can bound transferred average robustness from below by
\begin{align*}
\small
\underbrace{\Ex_{\substack{\cS\gets \tilde{D}^n\\ h\gets L(\cS)}}[\Rob_d(h,\hat{\cS})]}_{\textbf{Empirical robustness}} \:-\:  \underbrace{\big|\Ex_{\substack{\cS\gets \tilde{D}^n\\ h\gets L(\cS)}}[\Rob_d(h,\tilde{D}) \:-\: \Rob_d(h,\hat{\cS})]\big|}_{\textbf{Generalization penalty}} \:-\: \underbrace{\big|\Ex_{\substack{S\gets \tilde{D}^n\\ h\gets L(\cS)}}[\Rob_d(h,D)-\Rob_d(h,\tilde{D})]\big|}_{\textbf{Distribution-shift penalty}}.
\end{align*}

As the above decomposition suggests, in order to bound the average robustness, we need to bound both the generalization penalty and the distribution shift penalty.
The generalization penalty has been rigorously studied before in multiple works \citep{cullina2018pac,montasser2019vc,schmidt2018adversarially}. Hence, we focus on bounding the distribution shift penalty. Our goal is to provide a bound on the distribution-shift penalty that is independent of the classifier in hand and is only related to the properties of the distributions. With this goal, we define a notion of distance between two distributions.

\begin{definition}[Conditional Wasserstein distance]\label{def:cwd}
For two labeled distributions $D$ and $\tilde{D}$ supported on $\cX\times \cY$, we define $\cwd$ according to a distance metric $d$ as follows:
$$\cwd_d(D,\tilde{D})=\Ex_{(\cdot,y)\gets D}\left[\inf_{J\in \mathcal{J}(D\mid y,\tilde{D}\mid y)}\Ex_{(x,x')\gets J}[d(x,x')]\right]$$
where $\mathcal{J}(D,\tilde{D})$ is the set of joint distributions whose marginals are identical to $D$ and $\tilde{D}$. 
\end{definition}
It is simply the expectation of Wasserstein distance between conditional distributions for each class. 
Now, we are ready to state our main theorem that bounds the distribution shift penalty for any learning algorithm based only on the Wasserstein distance of the two distributions.

\begin{theorem}[Bounding distribution-shift penalty] \label{th: robustbound} 
Let $D$ and $\tilde{D}$ be two labeled distributions supported on $\cX\times \cY$ with identical label distributions, i.e., $\forall y^* \in \cY, \Pr_{(x,y)\gets D}[y=y^*] =  \Pr_{(x,y)\gets \tilde{D}}[y=y^*]$. Then for any classifier $h:\cX\to \cY$
$$|\Rob_d(h,\tilde{D}) - \Rob_d(h,D)| \leq  \cwd_d(D,\tilde{D}).$$
\end{theorem}

Theorem \ref{th: robustbound} shows how one can bound the distribution-shift penalty by minimizing the conditional Wasserstein distance between the two distributions. \bluetext{We provide its proof and empirical validation in Appendix~\ref{app: theory}}. \bluetext{Even if we successfully reduce the generalization penalty, the distribution-shift penalty may remain the dominant factor in  transferred robustness.} This theorem enables us to switch our attention from robust generalization to creating 
generative models for which the underlying distribution is close to the original distribution. \bluetext{In Appendix \ref{app: theory}, we also provide two other theorems about the tightness of Theorem \ref{th: robustbound} and the effect of combining clean and proxy data together.}

\subsection{Which \approxdist distribution to choose? - Approximating \cwdfull using robust discriminators} \label{sec: theory_arc}

Our analytical results suggest that a proxy distribution that has small \cwdfull to the \realdist can improve robustness. This means that optimizing for \cwdfull in a proxy distribution (e.g., a generative model) can potentially lead to improvements in adversarial robustness. Unfortunately, it is generally a hard task to empirically calculate the \cwdfull using only samples from the distributions \citep{mahloujifar2019empirically,bhagoji2019lowerBounds}. In this section, we introduce a metric named \emph{\aurpc} as a surrogate for \cwdfull that can be calculated when we only have sample access to the data distributions. Before defining our metric, we need to define the notion of \emph{robust discrimination} between two distributions.

\begin{wrapfigure}{r}{.4\textwidth}
    \vspace{-15pt}
    \begin{minipage}{\linewidth}
        \centering\captionsetup[subfigure]{justification=centering}
        \begin{minipage}{0.48\linewidth}
            \includegraphics[width=0.98\linewidth]{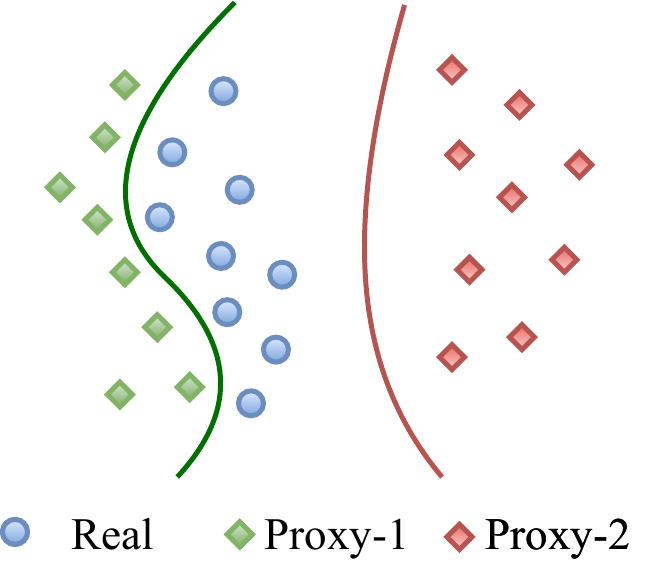}
            \subcaption{\scriptsize Original samples.}
            \label{fig: explain_blowup_base}
        \end{minipage}
        \begin{minipage}{0.48\linewidth}
            \includegraphics[width=0.98\linewidth]{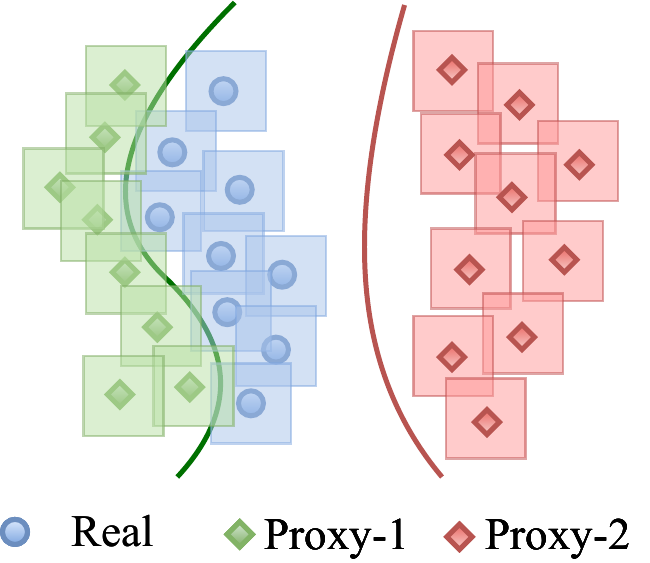}
            \subcaption{\scriptsize Perturbed samples.}
            \label{fig: explain_blowup_adv}
        \end{minipage}
    \end{minipage}
    \vspace{-5pt}
    \caption{\newbluetext{Why robust discrimination is effective.} (a) Without adversarial perturbations, a non-robust discriminator can perfectly distinguish both proxy distributions from the real distribution. (b) Under adversarial perturbations, classification between real and proxy-1 is much harder compared to real and proxy-2 distribution. The success of a robust discriminator is dependent on the proximity of real and proxy distributions.
    }
    \label{fig: explain_blowup}
    \vspace{-20pt}
\end{wrapfigure}
\noindent \textbf{Robust discrimination.} Our metric works based on how well a discriminator can \emph{robustly} distinguish between samples from the real and proxy distributions. Compared to a typical discriminator, the robust discriminator should be accurate even when there is an adversary who can perturb the sampled instance by a perturbation of size $\epsilon$. Intuitively, if there exists a successful and robust discriminator for two distributions $D$ and $\tilde{D}$, then it means that no adversary can make instances of $D$ to look like instances of $\tilde{D}$ even by making perturbations of size $\epsilon$. This means that $D$ and $\tilde{D}$ are far away from each other (Figure~\ref{fig: explain_blowup_adv}). \bluetext{ On the other hand, if no robust discriminator exists, then it means that the adversary can ''align`` most of the mass of proxy distribution with that of real distribution by making perturbations of size at most $\epsilon$.  This notion is closely related to the optimal transport between the proxy and real distribution. In particular, the optimal adversary here corresponds to a transport $J\in \cJ(D,\tilde{D})$ that minimizes $\Pr_{(x,x')\gets J(D,\tilde{D})}[d(x,x')\leq \epsilon]$ (Compare with Definition \ref{def:cwd}).}
To be more specific, the goal of a robust discriminator $\sigma$ is to maximize its robust accuracy\footnote{All the definitions of accuracy and robust accuracy in this section are defined for true distributions. However, in the experiments we work with their empirical variant.}:
\begin{equation*}
    \Racc_\epsilon (\sigma, D, \tilde{D}) = \tfrac{1}{2} \cdot \text{\normalfont Pr}_{x \gets D} \left[\forall x' \in Ball_\epsilon(x); \sigma(x')=1\right] + \tfrac{1}{2} \cdot \text{\normalfont Pr}_{x \gets \tilde{D}} \left[\forall x' \in \newbluetext{Ball_\epsilon(x)}; \sigma(x') =0\right]. 
\end{equation*}

 We use $\Racc_\epsilon^*(D,\tilde{D})$ to denote the accuracy of best existing robust discriminator, namely, $\Racc_\epsilon^*(D,\tilde{D})=\min_\sigma \Racc_\epsilon(\sigma, D,\tilde{D})$. Note that as $\epsilon$ grows from $0$ to $\infty$, the robust accuracy of any discriminator will decrease to at most $50\%$. If two distributions $D$ and $\tilde{D}$ are close to each other, this drop in accuracy happens much faster since the adversary can make images from two distributions to look alike by adding smaller perturbations. We use this intuition and define our metric \emph{\aurpc} as follows: $$\aurpcsf(D,\tilde{D})=\int_0^\infty (\Racc^*_\epsilon(D,\tilde{D}) - \frac{1}{2}) d\epsilon.$$
 
In other words, the \emph{\aurpc} metric is equal to the area under the robust discrimination accuracy v.s. perturbation curve (See Figure \ref{fig: aurpc} for a sample of this curve.) In Section \ref{sec:AURCP_calculation}, we discuss how this metric can be calculated with only samples from the \realdist and \approxdist distributions. \bluetext{When selecting a proxy distribution, we choose the one with the smallest \emph{\aurpc} value, since it is expected to be closest to \realdist distribution. We also demonstrate that \emph{\aurpc} metric is directly related to \cwdfull.} In particular, we have the following Theorem:
\begin{theorem} \label{thm:arc}
For any two distributions $\tilde{D}$ and $D$ with equal class probabilities we have $\cwd(D,\tilde{D})\geq 4\cdot \aurpcsf(D,\tilde{D})$.  Moreover, if for all labels $y$ , $(D\mid y)$ and $(\tilde{D}\mid y)$ are two concentric uniform spheres, then we have
$\cwd(D,\tilde{D}) = 4\cdot \aurpcsf(D,\tilde{D})$.
\end{theorem}

\noindent \textbf{Why non-robust discrimination is not an effective metric.} Deep neural networks are commonly used as a discriminator between synthetic and real images~\citep{goodfellow2014GAN, gui2020ganreview}. 
In our experiments, we find that even a four-layer convolutional network can achieve near 100\% classification accuracy on all selected proxy distributions (Appendix~\ref{app: nonrobust_disc}). This result shows that non-robust discrimination is not an effective measure of proximity (as highlighted in Figure~\ref{fig: explain_blowup_base}). 

\subsection{Bringing it all together: Improving adversarial robustness using \approxdist distributions} \label{sec: theory_sota}
\noindent \textbf{Overview of robust training.}
The key objective in robust training is to solve $\underset{\theta}{min} \underset{(x,y)\sim \dtrain}{\E} L_{adv}(\theta, x, y, \Omega)$, where $L_{adv}$ is determined by the training objective and $\Omega$ is the threat model. In adversarial training~\citep{madry2017towards, zhang2019tradeoff, carmon2019unlabeled, pang2021bagoftricks}, we train against projected gradient descent based (PGD) adversarial examples by selecting $L_{adv} (.) = \ell(\theta, {PGD}(x, \Omega), y)$.  In randomized smoothing~\citep{cohen2019certified, salman2019AdvSmooth, sehwag2020Hydra, carmon2019unlabeled}, where the objective is to achieve certified robustness~\citep{wong2018provable}, we aim to achieve invariance to random Gaussian Noise by selecting $L_{adv} (.) = \ell(\theta, x, y) + \beta D_{kl}(f_{\theta}(x), f_{\theta}(x + \delta))$, where $\delta \sim \mathrm{N}(0, \sigma^2I)$, and $D_{kl} (., .)$ is the KL-divergence. This loss function~\citep{zheng2016stabilityTrain}, referred to as robust stability training (RST), performs better than using $\ell(\theta,x+\delta,y)$~\citep{carmon2019unlabeled}. We refer the $L_{adv}$ in adversarial training and randomized smoothing as $L_{pgd}$ and $L_{smooth}$, respectively. We represent the cross-entropy loss function, which is used to train the classifier, as $\ell(.)$, where $\ell(\theta, x, y) = \langle-log(f_{\theta}(x)), \boldsymbol{y}\rangle$ and $\boldsymbol{y}$ is the one-hot vector encoding of $y$.

Our objective is to integrate synthetic samples from Proxy distributions in Robust Training (PORT). In particular, we use the following robust training formulation that combines samples from a real training dataset with synthetic samples.
\begin{align*} 
    \underset{\theta}{min}\:L_{agg};\:\:\:L_{agg} = \gamma \underset{(x,y)\sim \dtrain}{\E} L_{adv}(\theta, x, y, \Omega) + (1-\gamma) \underset{(x,y)\sim \dsyn}{\E} L_{adv}(\theta, x, y, \Omega)
\end{align*}
where $\gamma\in [0,1]$ and $L_{agg}$ is the aggregated adversarial loss over both $\dtrain$ and $\dsyn$. Depending on training objective, we select either $L_{pgd}$ or $L_{smooth}$ as $L_{adv}$. 
We approximate the aggregated loss using available training samples from the \realdist distribution ($\dtrain$) and a set of synthetic images sampled from the proxy distribution ($\dsyn$). We select $\gamma$ through a hyperparameter search. 


\noindent \textbf{\newtext{Which synthetic samples to choose?}} In particular, which set ($S$) of $N$ synthetic samples from the proxy distribution ($\tilde{D}$) leads to maximum transferred robustness on real data, i.e., $\Ex_{\substack{f\gets \hat{L}(S)}}[\Rob_d(f,D)]$, where $\hat{L}$ is a robust learning algorithm.
While it can be solved in many different ways, we follow a rejection sampling based approach. In particular, we measure proximity of each synthetic sample to \realdist distribution and use it as a metric to accept or reject a sample. We use output probability score from our trained \newbluetext{robust} discriminators as a measure of proximity. In particular we define $\textrm{\textit{synthetic-score}}(x) = \newbluetext{\min_{x_i \in {Ball_\epsilon (x)}}\sigma(x_i)}$ as the final score for each synthetic sample, where samples with lowest score are prioritized in the selection.

\section{Experimental results} \label{sec: experiments}
First, we show that our proposed approach of using proxy distributions brings significant improvements in the performance of both adversarial training and randomized smoothing across both $\ell_{\infty}$ and $\ell_2$ threat models (Section~\ref{sec: sota}).
\newtext{ We also uncover that diffusion-based generative models are more effective proxy distributions than generative adversarial networks (GANs) on multiple datasets. Next we delve deeper into why diffusion-based generative models are highly effective proxy distributions (Section~\ref{sec:AURCP_calculation}). We find that the key reason behind the success of diffusion based models is the that there is no \emph{robust} discriminator that can \emph{robustly} distinguish between the samples from diffusion based models and real data distribution. We demonstrate that our proposed metric (\emph{\aurpc}) provides a better characterization of the effectiveness of diffusion-based models than other widely used metrics. Finally, we investigate the effectiveness of adaptively selecting synthetic images in our framework.}


\subsection{Improving robustness using synthetic images from proxy distributions} \label{sec: sota}

\noindent \textit{Setup.}
\newtext{We consider five datasets, namely CIFAR-10~\citep{krizhevsky2014cifar}, CIFAR-100~\citep{krizhevsky2014cifar}, CelebA~\citep{liu2015CelebA}, AFHQ~\citep{choi2020starganv2}, and ImageNet~\citep{deng2009imagenet}. We also work with commonly used $l_{\infty}$ and $l_2$ threat models for adversarial attacks and use AutoAttack~\citep{croce2020autoattack} to measure robustness. We provide extensive details on the setup in Appendix~\ref{app: setup}.}

\noindent \textit{Evaluation metrics for robust training.} We use two key metrics to evaluate performance of trained models: \textit{clean accuracy} and \textit{robust accuracy}. While the former refers to the accuracy on unmodified test set images, the latter refers to the accuracy on adversarial examples generated from test set images. \bluetext{We refer to the robust accuracy of any discriminator trained to distinguish between synthetic and real data as robust discrimination accuracy.} We also use randomized smoothing to measure certified robust accuracy, i.e., robust accuracy under strongest possible adversary in the threat model.

\noindent \textit{Synthetic data.} \newtext{We sample synthetic images from a generative models, in particular from denoising diffusion-based probabilistic models (DDPM)~\citep{ho2020denoisingdiffusion, nichol2021improvedDdpm}. We provide extensive details on the number of images sampled and the sampling procedure in Appendix~\ref{app: setup}. We use $\gamma=0.4$ as it achieves best results (Appendix~\ref{app: sota}). We combine real and synthetic images in a $1$:$1$ ratio in each batch, thus irrespective of the number of synthetic images, our training time is only twice of the baseline methods.} 


\begin{table}[!htb]
    \centering
    \caption{\textbf{State-of-the-art adversarial robustness.} \bluetext{Comparing experimental results of our framework (\ours) with baselines for adversarial training on the CIFAR-10 dataset for both $\ell_{\infty}$ and $\ell_2$ threat model.} We sample synthetic images from the diffusion-based generative model (DDPM). 
    \textit{Clean}/\textit{Auto} refers to clean/robust accuracy measured with AutoAttack.}
    \label{tab: sota}
    \vspace{-5pt}
    \begin{subtable}[h]{0.49\textwidth}
        \centering
        \caption{$\ell_{\infty}$ threat model.}
        \renewcommand{\arraystretch}{1.2}
        \resizebox{0.99\linewidth}{!}{\begin{tabular}{ccccc}
        \toprule
        Method & Architecture & Parameters (M) & \textit{Clean} & \textit{Auto} \\ \midrule
        \citet{zhang2019tradeoff} & ResNet-18 & $11.2$ & $82.0$ & $48.7$ \\ 
        \ours & ResNet-18 & $11.2$ & $\boldsymbol{84.6}$ & $\boldsymbol{55.7}$ \\ \midrule
        \citet{rice2020overfitadv} & \wrn{34-20} & $184.5$ & $85.3$ & $53.4$ \\
        \citet{gowal2020uncoveringadv} & \wrn{70-16} & $266.8$ & $85.3$ & $57.2$
        \\ \midrule
        \citet{zhang2019tradeoff} & \wrn{34-10} & $46.2$ & $84.9$ & $53.1$ \\
        \ours  & \wrn{34-10} & $46.2$ & $\boldsymbol{87.0}$ & $\boldsymbol{60.6}$ \\
        \bottomrule
        \end{tabular}}
    \end{subtable} \hfill
    \begin{subtable}[h]{0.49\textwidth}
        \centering
        \caption{$\ell_2$ threat model.}
        \renewcommand{\arraystretch}{1.2}
        \resizebox{0.99\linewidth}{!}{\begin{tabular}{cccccc}
        \toprule
        Method & Architecture & Parameters (M) & \textit{Clean} & \textit{Auto} \\ \midrule
        \citet{rice2020overfitadv} & ResNet-18 & $11.2$ & $88.7$ & $67.7$ \\
        \ours & ResNet-18 & $11.2$ & $\boldsymbol{89.8}$ & $\boldsymbol{74.4}$
        \\ \midrule
        \citet{madry2017towards} & ResNet-50 & $23.5$ & $90.8$ & $69.2$ \\
        \citet{gowal2020uncoveringadv} & \wrn{70-16} & $266.8$ & $90.9$  & $74.5$
        \\ \midrule
        \citet{wu2020adversarial} & \wrn{34-10} & $46.2$ & $88.5$ & $73.7$ \\
        \ours & \wrn{34-10} & $46.2$ & $\boldsymbol{90.8}$ & $\boldsymbol{77.8}$
        \\ \bottomrule
        \end{tabular}}
     \end{subtable}
\end{table}

\subsubsection{Improving performance of adversarial training} \label{sec: sota_adv}
\textbf{State-of-the-art robust accuracy synthetic images (Table~\ref{tab: sota}, \ref{tab: sota_new}).} 
Using synthetic images from diffusion-based generative models (DDPM), our approach achieves state-of-the-art robust accuracy in the category of not using any extra real world data~\citep{croce2020robustbench}. \newtext{We improve robust accuracy by up to $7.5$\% and $6.7$\% over previous works in $\ell_{\infty}$ and $\ell_2$ threat model, respectively. Notably, the gain in performance is highest for the CIFAR-10 and Celeb-A datasets. Note that by using only synthetic images, our work also achieves competitive performance with previous works that use extra real-world images. For example, using additional real-data \citet{carmon2019unlabeled} achieve $59.53$\% robust accuracy on CIFAR-10 ($\ell_{\infty}$) while we achieve $60.6$\% robust accuracy by using only synthetic data.}

\begin{table}[!htb]
    \vspace{-5pt}
    \caption{\textbf{Similar improvement on additional datasets}. \newtext{Our approach further improves both clean and robust accuracy on each of datasets of Tables \ref{tab: sota} and \ref{tab: sota_new}. Unless a better baseline is available in previous works, we compare our results with baseline robust training approach from \citet{madry2017towards}.  \textit{Clean}/\textit{Auto} refers to clean/robust accuracy measured with AutoAttack. We provide results for AFHQ dataset in Table~\ref{tab: afhq_table} in Appendix.}}
    \vspace{-5pt}
    \label{tab: sota_new}
    \begin{subtable}{.35\linewidth}
        \centering
        \caption{CelebA ($64\times64$)}
        \renewcommand{\arraystretch}{1.3}
        \resizebox{\linewidth}{!}{
        \begin{tabular}{cccccc} \toprule
         & \multicolumn{2}{c}{$\ell_\infty (\epsilon=8/255)$} &  & \multicolumn{2}{c}{$\ell_2 (\epsilon=1.0)$} \\ \cmidrule{2-3} \cmidrule{5-6}
         & \textit{Clean} & \textit{Auto} &  & \textit{Clean} & \textit{Auto} \\ \midrule
        Baseline & $82.4$ & $60.2$ &  & $80.6$ & $58.9$ \\
        \ours & $84.8$ & $63.4$ &  & $82.3$ & $60.0$ \\
        $\Delta$ & $\boldsymbol{\texttt{+}2.4}$ & $\boldsymbol{\texttt{+}3.2}$	&  & $\boldsymbol{\texttt{+}1.7}$ & $\boldsymbol{\texttt{+}1.1}$ \\ \bottomrule
    \end{tabular}}
    \end{subtable} \hfill
    \begin{subtable}{.35\linewidth}
        \centering
        \caption{CIFAR-100 ($32\times32$)}
        \renewcommand{\arraystretch}{1.3}
        \resizebox{\linewidth}{!}{
        \begin{tabular}{cccccc} \toprule
         & \multicolumn{2}{c}{$\ell_\infty (\epsilon=8/255)$} &  & \multicolumn{2}{c}{$\ell_2 (\epsilon=0.5)$} \\ \cmidrule{2-3} \cmidrule{5-6}
         & \textit{Clean} & \textit{Auto} &  & \textit{Clean} & \textit{Auto} \\ \midrule
        Baseline & $58.4$ & $27.8$ &  & $62.9$ & $42.1$ \\
        \ours & $65.9$ & $31.2$ &  & $70.7$ & $47.8$ \\
        $\Delta$ & $\boldsymbol{\texttt{+}7.5}$ &  $\boldsymbol{\texttt{+}3.4}$	&  &  $\boldsymbol{\texttt{+}7.8}$ &  $\boldsymbol{\texttt{+}5.7}$ \\ \bottomrule
    \end{tabular}}
    \end{subtable} \hfill 
    \begin{subtable}{.23\linewidth}
    \centering
    \caption{ImageNet ($64\times64$)}
    \renewcommand{\arraystretch}{1.3}
    \tiny
    \resizebox{\linewidth}{!}{
    \begin{tabular}{ccc} \toprule
         & \multicolumn{2}{c}{$\ell_\infty (\epsilon=2/255)$} \\ \cmidrule{2-3}
         & \textit{Clean} & \textit{Auto}  \\ \midrule
        Baseline & $54.8$ & $28.1$\\
        \ours & $55.1$ & $28.4$ \\
        $\Delta$ & $\boldsymbol{\texttt{+}0.3}$ & $\boldsymbol{\texttt{+}0.3}$ \\ \bottomrule
    \end{tabular}}
\end{subtable}
\vspace{-10pt}
\end{table}


\begin{wraptable}{r}{0.4\linewidth}
    \centering
    \newbluetext{
    \caption{\newbluetext{\textbf{Additional baselines.} Using $\ell_\infty (\epsilon=8/255)$ attack for both datasets.}}
    \vspace{-5pt}
    \renewcommand{\arraystretch}{1.3}
    \tiny
    \resizebox{\linewidth}{!}{
    \begin{tabular}{ccc} \toprule
        CIFAR-100 & \textit{Clean} & \textit{Auto} \\ \midrule
        \citet{wu2020AdvWeightPerturb} & $60.4$ & $28.9$ \\
        \citet{rade2021helperBasedAdvTrain} & $\boldsymbol{61.5}$ & $28.9$ \\
        \citet{cui2021BoundaryGuideAdv} & $60.6$ & $29.3$ \\
        \ours & $60.6$ & $\boldsymbol{30.5}$ \\ \midrule
        CelebA & \textit{Clean} & \textit{Auto} \\ \midrule
        TRADES & $86.1$ & $61.2$ \\
        \ours & $86.1$ & $62.7$ \\
        $\Delta$  & $0.0$ & $\boldsymbol{1.5}$ \\ \bottomrule
    \end{tabular}}
    }
    \vspace{-15pt}
\end{wraptable}

\noindent \textbf{Proxy distribution offsets increase in network parameters.} We find that gains from using synthetic samples are equivalent to ones obtained by scaling network size by an order of magnitude (Table~\ref{tab: sota}). For example, a ResNet-18 network with synthetic data achieves higher robust accuracy ($\ell_{\infty}$) than a \wrn{34-20} trained without it, while having $16\times$ fewer parameters than the latter. Similar trend holds for \wrn{34-10} networks, when compared with a much larger \wrn{70-16} network. This trend also holds for both $\ell_{\infty}$ and $\ell_2$ threat models.

\noindent \textbf{Simultaneous improvement in clean accuracy.} Due to the accuracy vs robustness trade-off in adversarial training, improvement in robust accuracy often comes at the cost of clean accuracy. However synthetic samples provide a boost in both clean and robust accuracy, simultaneously. Using adaptively sampled synthetic images, we observe improvement in clean accuracy by up to $7.5$\% and $7.8$\% across $\ell_{\infty}$ and $\ell_2$ threat models, respectively. 

\newtext{\noindent \textbf{Comparison across generative models.} We find that synthetic samples from diffusion-based generative models~\citep{ho2020denoisingdiffusion, nichol2021improvedDdpm} provides significantly higher improvement in performance than generative adversarial networks (GANs) (Table~\ref{tab: gan_vs_ddpm}). We further investigate this phenomenon in the next section.
}

We further analyze sample complexity of adversarial training using synthetic data in Appendix~\ref{app: sample_complexity}. We also provide visual examples of real and synthetic images for each dataset at the end of the paper. 

\begin{figure*}[!htb]
	\centering
	\vspace{-5pt}
	\begin{minipage}{0.32\linewidth}
		\centering
		\captionof{table}{\newtext{\textbf{Comparing generative models.} When choosing proxy distribution in \ours ($\ell_{\infty}$), DDPM model outperforms the leading generative adversarial network (GAN) on each dataset.}}
		\label{tab: gan_vs_ddpm}
		\renewcommand{\arraystretch}{1.1}
		\resizebox{0.9\linewidth}{!}{
		\begin{tabular}{cccc} \toprule
            Dataset & \begin{tabular}[c]{@{}c@{}}Proxy \\ Distribution\end{tabular} & \textit{Clean} & \textit{Auto} \\ \midrule
            \multirow{3}{*}{CIFAR-10} & None & $84.9$ & $53.1$ \\ 
             & StyleGAN & $86.0$ & $52.5$ \\
             & DDPM & $\boldsymbol{87.0}$ & $\boldsymbol{60.6}$ \\ \midrule
            \multirow{3}{*}{CelebA} & None & $82.4$ & $60.2$ \\ 
             & StyleFormer & $84.1$ & $63.1$ \\
             & DDPM & $\boldsymbol{84.8}$ & $\boldsymbol{63.4}$ \\ \midrule
            \multirow{3}{*}{ImageNet} & None & $54.8$ & $28.1$ \\
             & BigGAN & $54.5$ & $28.2$ \\
             & DDPM & $\boldsymbol{55.1}$ & $\boldsymbol{28.4}$ \\ \bottomrule
        \end{tabular}}
	\end{minipage} \hfill
	\begin{minipage}{0.32\linewidth}
		\centering
	    \includegraphics[width=0.98\linewidth]{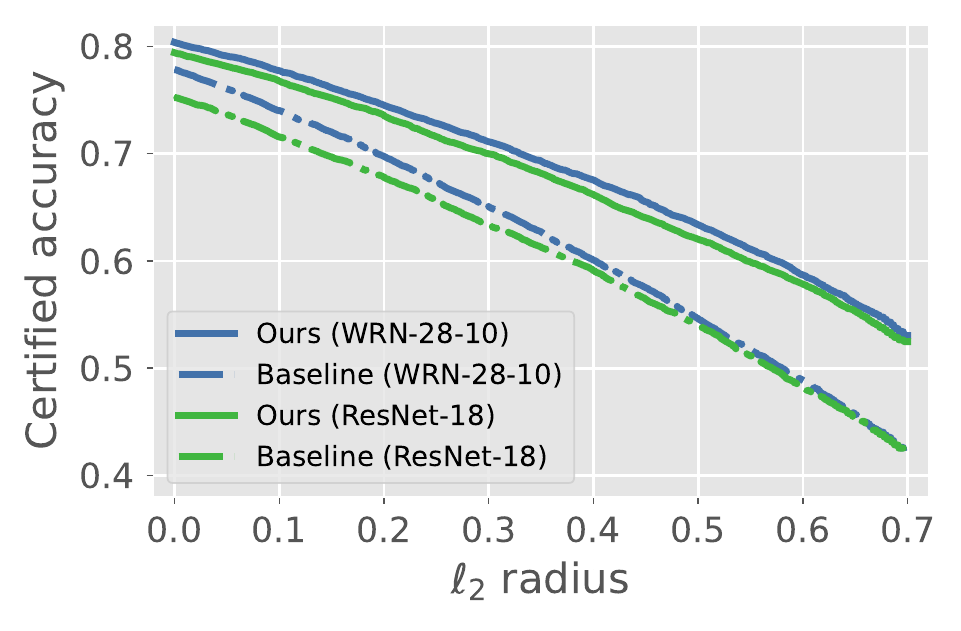}
        \caption{\textbf{Certified robustness.} Certified robust accuracy of baseline randomized smoothing technique, i.e., RST~\citep{carmon2019unlabeled} and our work at different perturbation budgets across two different network architectures.}
        \label{fig: vra}
	\end{minipage}
	\hfill
	\begin{minipage}{0.32\linewidth}
		\centering
		\captionof{table}{\textbf{Comparison of certified robustness.} Comparing clean accuracy (\textit{Clean}) and certified robust accuracy (\textit{Certified}) of our work with earlier approaches at $110/255$ $\ell_2$ perturbation budget (equivalent to $2/255$ $\ell_{\infty}$ perturbation budget) on the CIFAR-10 dataset.}
		\label{tab: vra}
		\renewcommand{\arraystretch}{1.1}
		\resizebox{\linewidth}{!}{
		\large
		\begin{tabular}{ccc}
		   \toprule
		   Method  &  \textit{Clean} & \textit{Certified} \\ \midrule
		   \citet{Wong2018ScalingProve} (single) & $68.3$ & $53.9$\\
		   \citet{Wong2018ScalingProve} (ensemble) & $64.1$ & $63.6$ \\
		   \newbluetext{CROWN-IBP (\cite{Zhang2020crownIbp})} & $71.5$ & $54.0$ \\
		   \citet{balunovic2019advPlusProve} & $78.4$ & $60.5$ \\
		   RST (\citet{carmon2019unlabeled}) & $77.9$ & $58.6$ \\
		   RST$_{500K}$ (\citet{carmon2019unlabeled}) &  $80.7$ & $63.8$ \\ 
		   \ours (ResNet-18) & $79.4$ & $\boldsymbol{64.6}$ \\
		   \ours (\wrn{28-10}) & $80.4$ & $\boldsymbol{66.2}$ \\
		   \bottomrule
		\end{tabular}}
	\end{minipage}
\end{figure*}




\subsubsection{Improving certified robustness with randomized smoothing} \label{sec: sota_certify}

We provide results on certified adversarial robustness in Table~\ref{tab: vra}.
We first compare the performance of our proposed approach with the baseline technique, i.e., RST~\citep{carmon2019unlabeled}. We achieve significantly higher certified robust accuracy than the baseline approach at all $\ell_2$ perturbations budgets for both ResNet-18 and \wrn{28-10} network architectures. Additionally, the robustness of our approach decays at a smaller rate than the baseline. At $\ell_{\infty}$ perturbation of $2/255$, equivalent to $\ell_2$ perturbation of $111/255$, our approach achieves $7.6$\% higher certified robust accuracy than RST. We also significantly outperform other certified robustness techniques which aren't based on randomized smoothing~\citep{Zhang2020crownIbp, Wong2018ScalingProve, balunovic2019advPlusProve}. Along with better certified robust accuracy, our approach also achieves better clean accuracy than most previous approaches, simultaneously. 

\noindent \textbf{Synthetic images vs real-world images.} 
Using only synthetic samples ($10$M), we also outperform RST where it uses an additional curated set of $500$K real-world images (RST$_{500K}$). While the latter achieves $63.8$\% certified robust accuracy, we improve it to $66.2$\%. \vikash{Discuss results when using only 500k synthetic images.}

\subsection{Delving deeper into use of proxy distributions in robust training}\label{sec:AURCP_calculation}
\newtext{Earlier in Section~\ref{sec: sota} we showed that diffusion-based generative models are an excellent choice for a proxy distribution for many datasets. Now we characterize the benefit of individual proxy distributions in our robust training framework. We consider seven different generative models, which are well-representative of the state-of-the-art on the CIFAR-10 dataset.  } 



\noindent \emph{Using synthetic-to-real transfer of robustness as ground truth.} From each generative model we sample one million synthetic images. We adversarially train a ResNet-18~\citep{he2016resnet} network on these images and rank generative models in order of robust accuracy achieved on the CIFAR-10 test set. We provide further experimental setup details in Appendix~\ref{app: aurpc}.

\noindent \emph{Calculating \aurpc.} At each perturbation budget ($\epsilon$), we adversarially train a ResNet-18 classifier to distinguish between synthetic and real images. We measure robust discrimination accuracy of trained classifier on a balanced held-out validation set. We calculate \aurpc by measuring the area under the robust discrimination accuracy and perturbation budget curve\footnote{We use \texttt{numpy.trapz(Racc-0.5, $\epsilon$)} to calculate \emph{\aurpc}.} (Figure~\ref{fig: aurpc}). We provide a detailed comparison of it with other metrics in Table~\ref{tab: gan_comaprison}.

\begin{figure*}[!htb]
	\centering
	\begin{minipage}{0.32\linewidth}
		\centering
        \includegraphics[width=0.9\linewidth]{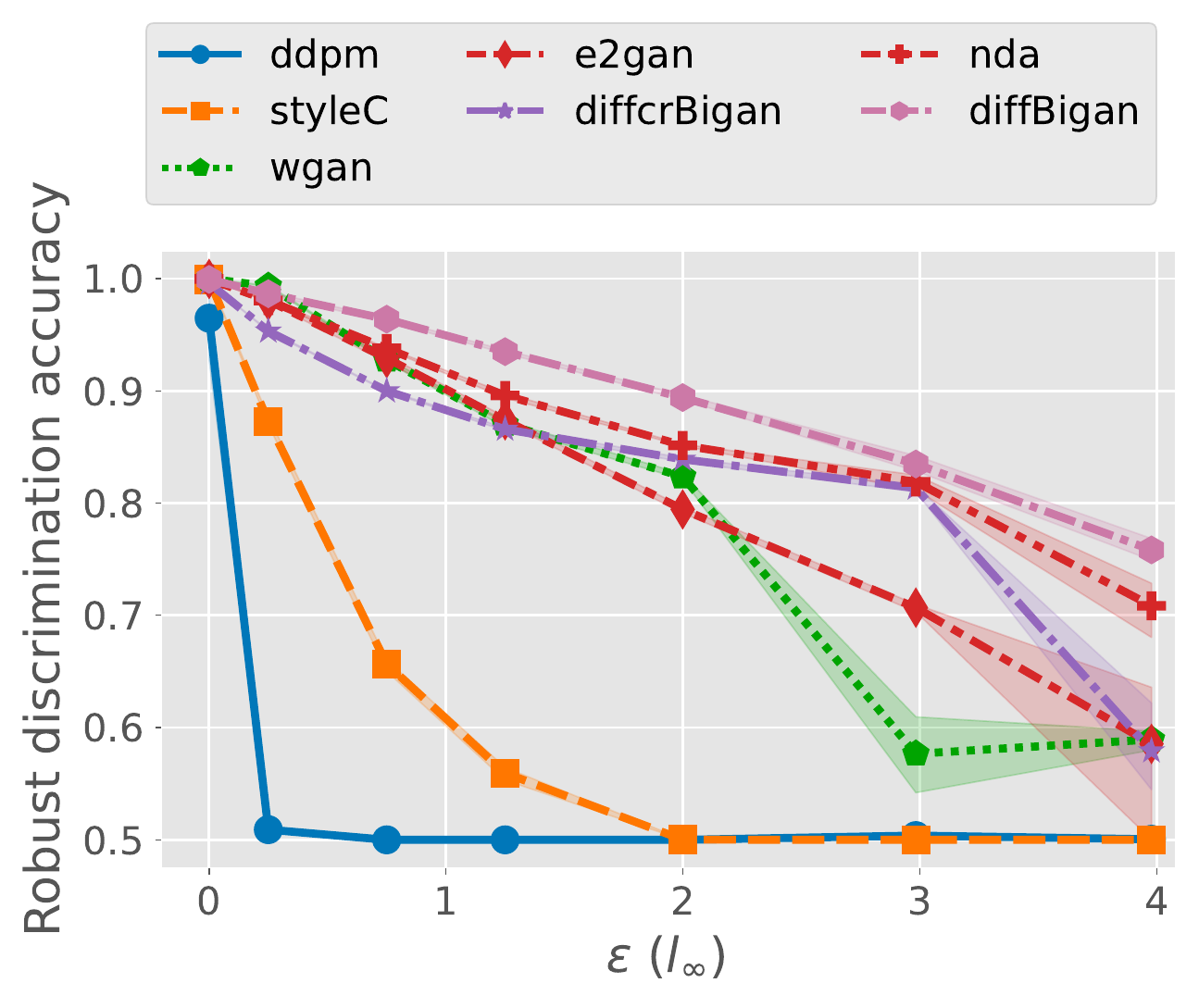}
        \caption{\bluetext{\textbf{Calculating \emph{\aurpc}.} For every generative model and perturbation budget ($\epsilon$), we first adversarially train a binary classifier on adversarial perturbed synthetic and CIFAR-10 images. Next we measure its robust discrimination accuracy on the validation set at the $\epsilon$ value used in training. \emph{\aurpc} is the area under the robust discrimination accuracy vs $\epsilon$ curve.}}
        \label{fig: aurpc}
	\end{minipage}
	\hfill
	\begin{minipage}{0.65\linewidth}
		\centering
		\captionof{table}{\textbf{Comparing different generative models.} \bluetext{We first adversarially train a ten class ResNet-18 model \textit{only} on $1$M synthetic images and measure its robust accuracy on the CIFAR-10 test set. We use this transferred robustness as a ground truth for benefit of each model. Next we match ranking predicted by each metric with the ground truth ranking. In contrast to all three three baseline, ranking from our proposed metric (\emph{\aurpc}) accurately matches the ground-truth.} \vikash{Use AA robust accuracy. Report FID/IS of our sampled images? Add other FID variations too. Color table header.}}
        \label{tab: gan_comaprison}
        \renewcommand{\arraystretch}{1.2}
        \resizebox{0.9\linewidth}{!}{
        \begin{tabular}{ccc||cccc}
            \toprule
            Rank & Model & \begin{tabular}[c]{@{}c@{}}Robust \\ accuracy\end{tabular} & \begin{tabular}[c]{@{}c@{}}FID ($\downarrow$)\end{tabular} & IS ($\uparrow$) & \oneNN ($\downarrow$) & \aurpctt ($\downarrow$) \\ \midrule
            -- & CIFAR-10 \footnote{\tiny Using 50,000 training images from the CIFAR-10 dataset.} & $48.7$ & --  & -- & -- & -- \\ \midrule
            $1$ & \ddpm (UC) \footnote{\tiny UC and C refers to unconditional and conditional generative models, respectively.} & $53.1$ & $3.17$ (2) & $9.46$ (2) &  $9.34$ (2) &  $0.06$ (\textbf{1}) \\
            $2$ & \styleC (C) & $45.0$ & $2.92$ (1) & $10.24$ (1) &  $9.42$ (3) & $0.32$ (\textbf{2})  \\
            $3$ & WGAN-ALP (UC) & $43.5$ & $12.96$ (7) & $8.34$ (7) & $10.10$ (7) & $1.09$ (\textbf{3})  \\ 
            $4$ & E2GAN (UC)  & $39.6$ & $11.26$ (5) & $8.51$ (5) & $8.96$ (1)  & $1.19$ (\textbf{4})  \\
            $5$ & DiffCrBigGAN (C) & $33.7$ & $4.30$ (3) & $9.17$ (3) &  $9.84$ (6)  & $1.30$ (\textbf{5})  \\ 
            $6$ & NDA  (C)& $33.4$ & $12.61$ (6) & $8.47$ (6) & $9.72$ (5) & $1.43$ (\textbf{6}) \\
            $7$ & DiffBigGAN (C) & $32.4$ & $4.61$ (4) & $9.16$ (4) &  $9.73$ (4) & $1.55$ (\textbf{7}) \\  \midrule
            \multicolumn{3}{c||}{Mean absolute ranking difference} & $1.7$ &  $1.7$ & $2.0$ & $\boldsymbol{0.0}$\\ \bottomrule
        \end{tabular}
        }
	\end{minipage}
\end{figure*}

\noindent \textbf{\textit{\aurpc} is highly effective at determining synthetic to real transfer of robustness (Table~\ref{tab: gan_comaprison}).}
\newtext{Each of the baseline metrics, including FID which is widely used, fail in accurately predicting the transfer of robustness from synthetic to real distributions. Even more, they also fail to identify DDPM~\citep{ho2020denoisingdiffusion} as the most successful proxy distribution.} Our proposed metric (\textit{\aurpc}), which measures the distinguishability of perturbed synthetic and real images, exactly predicts the ranking of different generative models in terms of the transfer of robustness. It also provides deeper insight into why synthetic samples from DDPM models are highly effective compared to other models. As shown in Figure~\ref{fig: aurpc}, discrimination between DDPM samples and real images becomes almost impossible even at very small perturbations budgets. A more detailed comparison is provided in Appendix~\ref{app: robust_disc}. 



\noindent \textbf{Adaptive sampling of synthetic data based on ARC (Appendix~\ref{appsubsec: adaptive}).} We compare the benefit of adaptive sampling with random sampling on CIFAR-10 dataset (for both $\ell_2$ and $\ell_{\infty}$ threat models). Our adaptive sampling achieves an average of $0.5\%$ and $0.2\%$ improvement in robust accuracy over random sampling in the $\ell_{\infty}$ and $\ell_{\infty}$ threat models, respectively.

\section{Related work}  \label{sec: relwork}
\noindent \textbf{Transfer of adversarial robustness.} This line of work focuses on the transfer of adversarial robustness, i.e., correct classification even under adversarial perturbations, when testing the model on different data distributions~\citep{shafahi2020AdvRobustTransfer, sehwag2019analyzing}. Note that this is different from just achieving correct classification on unmodified images across different distributions~\citep{taori2020NatShiftRobust, hendrycks2019CommonCorrupt}. Here we provide a theoretical analysis of the transfer of adversarial robustness between data distributions. 

\noindent \textbf{Using extra curated real-world data in robust training.} Prior works \citep{zhai2019advJustmoreData, carmon2019unlabeled, alayrac2019unsupadv, najafi2019advIncompleteData, deng2020AdvExtraOutDomain} have argued for using more training data in adversarial training and often resort to curating additional real-world samples. In contrast, we model the proxy distribution from the limited training images available and sample additional synthetic images from this distribution.

\noindent \textbf{Generative models for proxy distributions.}
State-of-the-art generative models are capable of modeling the distribution of current large-scale image datasets. In particular, generative adversarial networks (GANs) have excelled at this task~\citep{goodfellow2014GAN, karras2020styleganAda, gui2020ganreview}. Though GANs generate images with high fidelity, they often lack high diversity~\citep{ravuri2019CAS}. However, samples from recently proposed diffusion process based models achieve both high diversity and fidelity~\citep{ho2020denoisingdiffusion, nichol2021improvedDdpm}. Fréchet Inception Distance (FID)~\citep{Heusel2017FID} and Inception Score (IS)~\citep{Heusel2017FID} are two common metrics to evaluate the quality of samples from generative models.

\noindent \textbf{Using generative models to improve adversarial robustness.} Earlier works have used generative models to learn training data manifold and the use it to map input samples to data manifold~\citep{samangouei2018defenseGan, jalal2017robustmanifold, Xu2018featsqueeze}. However, most of these techniques are broken against adaptive attacks~\citep{athalye2018obfuscatedGrad, Tramer2020AdaptiveAttack}. We use generative models to sample additional training samples which further improve the adversarial robustness. 

\noindent \textbf{Comparison with \citet{rebuffi2021fixingaugmentation}.} A concurrent work by ~\citet{rebuffi2021fixingaugmentation} also uses randomly sampled synthetic images to improve adversarial robustness. In comparison, we provide a theoretical analysis of when synthetic data helps along with metrics to optimize the selection of both generative models and individual samples. We also demonstrate an improvement in certified robust accuracy using synthetic samples. Despite these differences, similar benefits of using proxy distributions in two independent and concurrent works further ascertain the importance of this research direction.

\section{Discussion and Broader Outlook} \label{sec: discussion}
Using synthetic data has been a compelling solution in many applications, such as healthcare~\citep{jordon2018pateGan} and autonomous driving~\citep{mayer2016synDriving} since it makes collecting a large amount of data feasible. In a similar spirit, we use synthetic data to make deep neural networks more robust against adversarial attacks. \bluetext{We investigate foundational questions such as determining the transfer of robustness from synthetic to real data and determining selection criteria for how to choose generative models or individual samples.}
\bluetext{Finally, we note that while it is crucial to improve robustness against the threat of adversarial examples, it also has an unwanted side-effect in domains where adversarial examples are used for good. Recent works use them to provide privacy on the web~\citep{shan2020fawkes} or to evade website fingerprinting methods~\citep{rahman2020mockingbird}. Improving defenses against adversarial examples will negatively hurt these applications.}

\section{Reproducibility}
We provide proof of each of our theorems in Appendix~\ref{app: theory}. Similarly, we provide extensive details on our experimental setup in Appendix~\ref{app: setup}. Our work is also featured on RobustBench~\citep{croce2020robustbench}, an external benchmark for standardized evaluation of adversarial robustness. Our robustly trained models are also available through the RobustBench API. For further reproducibility, we have also submitted our code with the supplementary material.

\bibliography{ref}

\begin{thebibliography}{82}
\providecommand{\natexlab}[1]{#1}
\providecommand{\url}[1]{\texttt{#1}}
\expandafter\ifx\csname urlstyle\endcsname\relax
  \providecommand{\doi}[1]{doi: #1}\else
  \providecommand{\doi}{doi: \begingroup \urlstyle{rm}\Url}\fi

\bibitem[Athalye et~al.(2018)Athalye, Carlini, and
  Wagner]{athalye2018obfuscatedGrad}
Anish Athalye, Nicholas Carlini, and David Wagner.
\newblock Obfuscated gradients give a false sense of security: Circumventing
  defenses to adversarial examples.
\newblock In \emph{International Conference on Machine Learning}, pp.\
  274--283. PMLR, 2018.

\bibitem[Balunovic \& Vechev(2019)Balunovic and
  Vechev]{balunovic2019advPlusProve}
Mislav Balunovic and Martin Vechev.
\newblock Adversarial training and provable defenses: Bridging the gap.
\newblock In \emph{International Conference on Learning Representations}, 2019.

\bibitem[Bhagoji et~al.(2019)Bhagoji, Cullina, and
  Mittal]{bhagoji2019lowerBounds}
Arjun~Nitin Bhagoji, Daniel Cullina, and Prateek Mittal.
\newblock Lower bounds on adversarial robustness from optimal transport.
\newblock In \emph{Advances in Neural Information Processing Systems}, pp.\
  7496--7508, 2019.

\bibitem[Biggio \& Roli(2018)Biggio and Roli]{biggio2018wildPatterns}
Battista Biggio and Fabio Roli.
\newblock Wild patterns: Ten years after the rise of adversarial machine
  learning.
\newblock volume~84, pp.\  317--331. Elsevier, 2018.

\bibitem[Biggio et~al.(2013)Biggio, Corona, Maiorca, Nelson, {\v{S}}rndi{\'c},
  Laskov, Giacinto, and Roli]{biggio2013evasion}
Battista Biggio, Igino Corona, Davide Maiorca, Blaine Nelson, Nedim
  {\v{S}}rndi{\'c}, Pavel Laskov, Giorgio Giacinto, and Fabio Roli.
\newblock Evasion attacks against machine learning at test time.
\newblock In \emph{Joint European conference on machine learning and knowledge
  discovery in databases}, pp.\  387--402. Springer, 2013.

\bibitem[Brock et~al.(2019)Brock, Donahue, and Simonyan]{brock2018bigGandeep}
Andrew Brock, Jeff Donahue, and Karen Simonyan.
\newblock Large scale gan training for high fidelity natural image synthesis.
\newblock In \emph{International Conference on Learning Representations}, 2019.

\bibitem[Carmon et~al.(2019)Carmon, Raghunathan, Schmidt, Duchi, and
  Liang]{carmon2019unlabeled}
Yair Carmon, Aditi Raghunathan, Ludwig Schmidt, John~C Duchi, and Percy~S
  Liang.
\newblock Unlabeled data improves adversarial robustness.
\newblock In \emph{Advances in Neural Information Processing Systems}, pp.\
  11190--11201, 2019.

\bibitem[Choi et~al.(2020)Choi, Uh, Yoo, and Ha]{choi2020starganv2}
Yunjey Choi, Youngjung Uh, Jaejun Yoo, and Jung-Woo Ha.
\newblock Stargan v2: Diverse image synthesis for multiple domains.
\newblock In \emph{Proceedings of the IEEE/CVF Conference on Computer Vision
  and Pattern Recognition}, pp.\  8188--8197, 2020.

\bibitem[Cohen et~al.(2019)Cohen, Rosenfeld, and Kolter]{cohen2019certified}
Jeremy Cohen, Elan Rosenfeld, and Zico Kolter.
\newblock Certified adversarial robustness via randomized smoothing.
\newblock In \emph{International Conference on Machine Learning}, pp.\
  1310--1320. PMLR, 2019.

\bibitem[Croce \& Hein(2020)Croce and Hein]{croce2020autoattack}
Francesco Croce and Matthias Hein.
\newblock Reliable evaluation of adversarial robustness with an ensemble of
  diverse parameter-free attacks.
\newblock In \emph{International Conference on Machine Learning}, pp.\
  2206--2216. PMLR, 2020.

\bibitem[Croce et~al.(2020)Croce, Andriushchenko, Sehwag, Flammarion, Chiang,
  Mittal, and Hein]{croce2020robustbench}
Francesco Croce, Maksym Andriushchenko, Vikash Sehwag, Nicolas Flammarion, Mung
  Chiang, Prateek Mittal, and Matthias Hein.
\newblock Robustbench: a standardized adversarial robustness benchmark.
\newblock \emph{arXiv preprint arXiv:2010.09670}, 2020.

\bibitem[Cui et~al.(2021)Cui, Liu, Wang, and Jia]{cui2021BoundaryGuideAdv}
Jiequan Cui, Shu Liu, Liwei Wang, and Jiaya Jia.
\newblock Learnable boundary guided adversarial training.
\newblock In \emph{Proceedings of the IEEE/CVF International Conference on
  Computer Vision}, pp.\  15721--15730, 2021.

\bibitem[Cullina et~al.(2018)Cullina, Bhagoji, and Mittal]{cullina2018pac}
Daniel Cullina, Arjun~Nitin Bhagoji, and Prateek Mittal.
\newblock Pac-learning in the presence of adversaries.
\newblock volume~31, pp.\  230--241, 2018.

\bibitem[Deng et~al.(2009)Deng, Dong, Socher, Li, Li, and
  Fei-Fei]{deng2009imagenet}
Jia Deng, Wei Dong, Richard Socher, Li-Jia Li, Kai Li, and Li~Fei-Fei.
\newblock Imagenet: A large-scale hierarchical image database.
\newblock In \emph{2009 IEEE conference on computer vision and pattern
  recognition}, pp.\  248--255. Ieee, 2009.

\bibitem[Deng et~al.(2021)Deng, Zhang, Ghorbani, and
  Zou]{deng2020AdvExtraOutDomain}
Zhun Deng, Linjun Zhang, Amirata Ghorbani, and James Zou.
\newblock Improving adversarial robustness via unlabeled out-of-domain data.
\newblock In \emph{International Conference on Artificial Intelligence and
  Statistics}, pp.\  2845--2853. PMLR, 2021.

\bibitem[Diochnos et~al.(2018)Diochnos, Mahloujifar, and
  Mahmoody]{diochnos2018adversarial}
Dimitrios~I Diochnos, Saeed Mahloujifar, and Mohammad Mahmoody.
\newblock Adversarial risk and robustness: general definitions and implications
  for the uniform distribution.
\newblock In \emph{Proceedings of the 32nd International Conference on Neural
  Information Processing Systems}, pp.\  10380--10389, 2018.

\bibitem[Dowson \& Landau(1982)Dowson and Landau]{dowson1982frechetDistance}
DC~Dowson and BV~Landau.
\newblock The fr{\'e}chet distance between multivariate normal distributions.
\newblock In \emph{Journal of multivariate analysis}, volume~12, pp.\
  450--455, 1982.

\bibitem[Gilmer et~al.(2018)Gilmer, Metz, Faghri, Schoenholz, Raghu,
  Wattenberg, and Goodfellow]{gilmer2018adversarial}
Justin Gilmer, Luke Metz, Fartash Faghri, Samuel~S Schoenholz, Maithra Raghu,
  Martin Wattenberg, and Ian Goodfellow.
\newblock Adversarial spheres.
\newblock \emph{arXiv preprint arXiv:1801.02774}, 2018.

\bibitem[Givens et~al.(1984)Givens, Shortt, et~al.]{givens1984classWasserstein}
Clark~R Givens, Rae~Michael Shortt, et~al.
\newblock A class of wasserstein metrics for probability distributions.
\newblock In \emph{The Michigan Mathematical Journal}, volume~31, pp.\
  231--240, 1984.

\bibitem[Goodfellow et~al.(2014)Goodfellow, Pouget-Abadie, Mirza, Xu,
  Warde-Farley, Ozair, Courville, and Bengio]{goodfellow2014GAN}
Ian Goodfellow, Jean Pouget-Abadie, Mehdi Mirza, Bing Xu, David Warde-Farley,
  Sherjil Ozair, Aaron Courville, and Yoshua Bengio.
\newblock Generative adversarial nets.
\newblock In \emph{Advances in Neural Information Processing Systems},
  volume~27, 2014.

\bibitem[Gowal et~al.(2020)Gowal, Qin, Uesato, Mann, and
  Kohli]{gowal2020uncoveringadv}
Sven Gowal, Chongli Qin, Jonathan Uesato, Timothy Mann, and Pushmeet Kohli.
\newblock Uncovering the limits of adversarial training against norm-bounded
  adversarial examples.
\newblock \emph{arXiv preprint arXiv:2010.03593}, 2020.

\bibitem[Gui et~al.(2020)Gui, Sun, Wen, Tao, and Ye]{gui2020ganreview}
Jie Gui, Zhenan Sun, Yonggang Wen, Dacheng Tao, and Jieping Ye.
\newblock A review on generative adversarial networks: Algorithms, theory, and
  applications.
\newblock \emph{arXiv preprint arXiv:2001.06937}, 2020.

\bibitem[He et~al.(2016)He, Zhang, Ren, and Sun]{he2016resnet}
Kaiming He, Xiangyu Zhang, Shaoqing Ren, and Jian Sun.
\newblock Deep residual learning for image recognition.
\newblock In \emph{Proceedings of the IEEE conference on computer vision and
  pattern recognition}, pp.\  770--778, 2016.

\bibitem[Hendrycks \& Dietterich(2019)Hendrycks and
  Dietterich]{hendrycks2019CommonCorrupt}
Dan Hendrycks and Thomas~G. Dietterich.
\newblock Benchmarking neural network robustness to common corruptions and
  perturbations.
\newblock In \emph{7th International Conference on Learning Representations},
  2019.

\bibitem[Heusel et~al.(2017)Heusel, Ramsauer, Unterthiner, Nessler, and
  Hochreiter]{Heusel2017FID}
Martin Heusel, Hubert Ramsauer, Thomas Unterthiner, Bernhard Nessler, and Sepp
  Hochreiter.
\newblock Gans trained by a two time-scale update rule converge to a local nash
  equilibrium.
\newblock In \emph{Advances in neural information processing systems},
  volume~30, 2017.

\bibitem[Ho et~al.(2020)Ho, Jain, and Abbeel]{ho2020denoisingdiffusion}
Jonathan Ho, Ajay Jain, and Pieter Abbeel.
\newblock Denoising diffusion probabilistic models.
\newblock In \emph{Advances in Neural Information Processing Systems},
  volume~33, pp.\  6840--6851, 2020.

\bibitem[Jalal et~al.(2017)Jalal, Ilyas, Daskalakis, and
  Dimakis]{jalal2017robustmanifold}
Ajil Jalal, Andrew Ilyas, Constantinos Daskalakis, and Alexandros~G Dimakis.
\newblock The robust manifold defense: Adversarial training using generative
  models.
\newblock \emph{arXiv preprint arXiv:1712.09196}, 2017.

\bibitem[Jordon et~al.(2018)Jordon, Yoon, and Van
  Der~Schaar]{jordon2018pateGan}
James Jordon, Jinsung Yoon, and Mihaela Van Der~Schaar.
\newblock Pate-gan: Generating synthetic data with differential privacy
  guarantees.
\newblock In \emph{International Conference on Learning Representations}, 2018.

\bibitem[Karras et~al.(2020)Karras, Aittala, Hellsten, Laine, Lehtinen, and
  Aila]{karras2020styleganAda}
Tero Karras, Miika Aittala, Janne Hellsten, Samuli Laine, Jaakko Lehtinen, and
  Timo Aila.
\newblock Training generative adversarial networks with limited data.
\newblock In \emph{Advances in Neural Information Processing Systems},
  volume~33, pp.\  12104--12114, 2020.

\bibitem[Kolesnikov et~al.(2020)Kolesnikov, Beyer, Zhai, Puigcerver, Yung,
  Gelly, and Houlsby]{Kolesnikov2020BiT}
Alexander Kolesnikov, Lucas Beyer, Xiaohua Zhai, Joan Puigcerver, Jessica Yung,
  Sylvain Gelly, and Neil Houlsby.
\newblock Big transfer (bit): General visual representation learning.
\newblock In \emph{ECCV}, 2020.

\bibitem[Krizhevsky et~al.(2014)Krizhevsky, Nair, and
  Hinton]{krizhevsky2014cifar}
Alex Krizhevsky, Vinod Nair, and Geoffrey Hinton.
\newblock The cifar-10 dataset.
\newblock \emph{online: http://www. cs. toronto. edu/kriz/cifar. html}, 2014.

\bibitem[LeCun \& Cortes(2010)LeCun and Cortes]{lecun2010mnist}
Yann LeCun and Corinna Cortes.
\newblock {MNIST} handwritten digit database.
\newblock \emph{AT\&T Labs [Online]. Available: http://yann. lecun.
  com/exdb/mnist}, 2010.

\bibitem[Liu et~al.(2015)Liu, Luo, Wang, and Tang]{liu2015CelebA}
Ziwei Liu, Ping Luo, Xiaogang Wang, and Xiaoou Tang.
\newblock Deep learning face attributes in the wild.
\newblock In \emph{Proceedings of the IEEE international conference on computer
  vision}, pp.\  3730--3738, 2015.

\bibitem[Madry et~al.(2018)Madry, Makelov, Schmidt, Tsipras, and
  Vladu]{madry2017towards}
Aleksander Madry, Aleksandar Makelov, Ludwig Schmidt, Dimitris Tsipras, and
  Adrian Vladu.
\newblock Towards deep learning models resistant to adversarial attacks.
\newblock \emph{International Conference on Learning Representations}, 2018.

\bibitem[Mahajan et~al.(2018)Mahajan, Girshick, Ramanathan, He, Paluri, Li,
  Bharambe, and Van Der~Maaten]{mahajan2018BillionWSL}
Dhruv Mahajan, Ross Girshick, Vignesh Ramanathan, Kaiming He, Manohar Paluri,
  Yixuan Li, Ashwin Bharambe, and Laurens Van Der~Maaten.
\newblock Exploring the limits of weakly supervised pretraining.
\newblock In \emph{Proceedings of the European Conference on Computer Vision
  (ECCV)}, pp.\  181--196, 2018.

\bibitem[Mahloujifar et~al.(2019{\natexlab{a}})Mahloujifar, Diochnos, and
  Mahmoody]{mahloujifar2019curse}
Saeed Mahloujifar, Dimitrios~I Diochnos, and Mohammad Mahmoody.
\newblock The curse of concentration in robust learning: Evasion and poisoning
  attacks from concentration of measure.
\newblock In \emph{Proceedings of the AAAI Conference on Artificial
  Intelligence}, volume~33, pp.\  4536--4543, 2019{\natexlab{a}}.

\bibitem[Mahloujifar et~al.(2019{\natexlab{b}})Mahloujifar, Zhang, Mahmoody,
  and Evans]{mahloujifar2019empirically}
Saeed Mahloujifar, Xiao Zhang, Mohammad Mahmoody, and David Evans.
\newblock Empirically measuring concentration: Fundamental limits on intrinsic
  robustness.
\newblock \emph{arXiv preprint arXiv:1905.12202}, 2019{\natexlab{b}}.

\bibitem[Mayer et~al.(2016)Mayer, Ilg, Hausser, Fischer, Cremers, Dosovitskiy,
  and Brox]{mayer2016synDriving}
Nikolaus Mayer, Eddy Ilg, Philip Hausser, Philipp Fischer, Daniel Cremers,
  Alexey Dosovitskiy, and Thomas Brox.
\newblock A large dataset to train convolutional networks for disparity,
  optical flow, and scene flow estimation.
\newblock In \emph{Proceedings of the IEEE conference on computer vision and
  pattern recognition}, pp.\  4040--4048, 2016.

\bibitem[Montasser et~al.(2019)Montasser, Hanneke, and Srebro]{montasser2019vc}
Omar Montasser, Steve Hanneke, and Nathan Srebro.
\newblock Vc classes are adversarially robustly learnable, but only improperly.
\newblock In \emph{Conference on Learning Theory}, pp.\  2512--2530. PMLR,
  2019.

\bibitem[Najafi et~al.(2019)Najafi, Maeda, Koyama, and
  Miyato]{najafi2019advIncompleteData}
Amir Najafi, Shin-ichi Maeda, Masanori Koyama, and Takeru Miyato.
\newblock Robustness to adversarial perturbations in learning from incomplete
  data.
\newblock In \emph{Proceedings of the 33rd International Conference on Neural
  Information Processing Systems}, pp.\  5541--5551, 2019.

\bibitem[Nakkiran et~al.(2021)Nakkiran, Neyshabur, and
  Sedghi]{nakkiran2021deepbootstrap}
Preetum Nakkiran, Behnam Neyshabur, and Hanie Sedghi.
\newblock The bootstrap framework: Generalization through the lens of online
  optimization.
\newblock In \emph{International Conference on Learning Representations}, 2021.

\bibitem[Nichol \& Dhariwal(2021)Nichol and Dhariwal]{nichol2021improvedDdpm}
Alexander~Quinn Nichol and Prafulla Dhariwal.
\newblock Improved denoising diffusion probabilistic models.
\newblock In \emph{Proceedings of the 38th International Conference on Machine
  Learning}, volume 139, pp.\  8162--8171. PMLR, 2021.

\bibitem[Pang et~al.(2021)Pang, Yang, Dong, Su, and Zhu]{pang2021bagoftricks}
Tianyu Pang, Xiao Yang, Yinpeng Dong, Hang Su, and Jun Zhu.
\newblock Bag of tricks for adversarial training.
\newblock In \emph{International Conference on Learning Representations}, 2021.

\bibitem[Park \& Kim(2021)Park and Kim]{park2021styleformer}
Jeeseung Park and Younggeun Kim.
\newblock Styleformer: Transformer based generative adversarial networks with
  style vector.
\newblock \emph{arXiv preprint arXiv:2106.07023}, 2021.

\bibitem[Rade \& Moosavi-Dezfooli(2021)Rade and
  Moosavi-Dezfooli]{rade2021helperBasedAdvTrain}
Rahul Rade and Seyed-Mohsen Moosavi-Dezfooli.
\newblock Helper-based adversarial training: Reducing excessive margin to
  achieve a better accuracy vs. robustness trade-off.
\newblock In \emph{ICML 2021 Workshop on Adversarial Machine Learning}, 2021.

\bibitem[Rahman et~al.(2020)Rahman, Imani, Mathews, and
  Wright]{rahman2020mockingbird}
Mohammad~Saidur Rahman, Mohsen Imani, Nate Mathews, and Matthew Wright.
\newblock Mockingbird: Defending against deep-learning-based website
  fingerprinting attacks with adversarial traces.
\newblock volume~16, pp.\  1594--1609. IEEE, 2020.

\bibitem[Ravuri \& Vinyals(2019)Ravuri and Vinyals]{ravuri2019CAS}
Suman Ravuri and Oriol Vinyals.
\newblock Classification accuracy score for conditional generative models.
\newblock In \emph{Advances in Neural Information Processing Systems},
  volume~32, 2019.

\bibitem[Rebuffi et~al.(2021)Rebuffi, Gowal, Calian, Stimberg, Wiles, and
  Mann]{rebuffi2021fixingaugmentation}
Sylvestre-Alvise Rebuffi, Sven Gowal, Dan~A Calian, Florian Stimberg, Olivia
  Wiles, and Timothy Mann.
\newblock Fixing data augmentation to improve adversarial robustness.
\newblock \emph{arXiv preprint arXiv:2103.01946}, 2021.

\bibitem[Recht et~al.(2019)Recht, Roelofs, Schmidt, and
  Shankar]{recht2019ImageNetv2}
Benjamin Recht, Rebecca Roelofs, Ludwig Schmidt, and Vaishaal Shankar.
\newblock Do imagenet classifiers generalize to imagenet?
\newblock In \emph{International Conference on Machine Learning}, pp.\
  5389--5400. PMLR, 2019.

\bibitem[Rice et~al.(2020)Rice, Wong, and Kolter]{rice2020overfitadv}
Leslie Rice, Eric Wong, and Zico Kolter.
\newblock Overfitting in adversarially robust deep learning.
\newblock In \emph{International Conference on Machine Learning}, pp.\
  8093--8104. PMLR, 2020.

\bibitem[Salman et~al.(2019)Salman, Li, Razenshteyn, Zhang, Zhang, Bubeck, and
  Yang]{salman2019AdvSmooth}
Hadi Salman, Jerry Li, Ilya~P. Razenshteyn, Pengchuan Zhang, Huan Zhang,
  S{\'{e}}bastien Bubeck, and Greg Yang.
\newblock Provably robust deep learning via adversarially trained smoothed
  classifiers.
\newblock In \emph{32 Annual Conference on Neural Information Processing
  Systems}, pp.\  11289--11300, 2019.

\bibitem[Samangouei et~al.(2018)Samangouei, Kabkab, and
  Chellappa]{samangouei2018defenseGan}
Pouya Samangouei, Maya Kabkab, and Rama Chellappa.
\newblock Defense-gan: Protecting classifiers against adversarial attacks using
  generative models.
\newblock In \emph{International Conference on Learning Representations}, 2018.

\bibitem[Schmidt et~al.(2018{\natexlab{a}})Schmidt, Santurkar, Tsipras, Talwar,
  and Madry]{schmidt2018AdvMoreData}
Ludwig Schmidt, Shibani Santurkar, Dimitris Tsipras, Kunal Talwar, and
  Aleksander Madry.
\newblock Adversarially robust generalization requires more data.
\newblock In \emph{Advances in Neural Information Processing Systems},
  volume~31, 2018{\natexlab{a}}.

\bibitem[Schmidt et~al.(2018{\natexlab{b}})Schmidt, Santurkar, Tsipras, Talwar,
  and Madry]{schmidt2018adversarially}
Ludwig Schmidt, Shibani Santurkar, Dimitris Tsipras, Kunal Talwar, and
  Aleksander Madry.
\newblock Adversarially robust generalization requires more data.
\newblock \emph{arXiv preprint arXiv:1804.11285}, 2018{\natexlab{b}}.

\bibitem[Sehwag et~al.(2019)Sehwag, Bhagoji, Song, Sitawarin, Cullina, Chiang,
  and Mittal]{sehwag2019analyzing}
Vikash Sehwag, Arjun~Nitin Bhagoji, Liwei Song, Chawin Sitawarin, Daniel
  Cullina, Mung Chiang, and Prateek Mittal.
\newblock Analyzing the robustness of open-world machine learning.
\newblock In \emph{Proceedings of the 12th ACM Workshop on Artificial
  Intelligence and Security}, pp.\  105--116, 2019.

\bibitem[Sehwag et~al.(2020)Sehwag, Wang, Mittal, and Jana]{sehwag2020Hydra}
Vikash Sehwag, Shiqi Wang, Prateek Mittal, and Suman Jana.
\newblock Hydra: Pruning adversarially robust neural networks.
\newblock In \emph{Advances in Neural Information Processing Systems},
  volume~33, pp.\  19655--19666, 2020.

\bibitem[Shafahi et~al.(2019)Shafahi, Najibi, Ghiasi, Xu, Dickerson, Studer,
  Davis, Taylor, and Goldstein]{shafahi2019freeAdv}
Ali Shafahi, Mahyar Najibi, Mohammad~Amin Ghiasi, Zheng Xu, John Dickerson,
  Christoph Studer, Larry~S Davis, Gavin Taylor, and Tom Goldstein.
\newblock Adversarial training for free!
\newblock In \emph{Advances in Neural Information Processing Systems}, 2019.

\bibitem[Shafahi et~al.(2020)Shafahi, Saadatpanah, Zhu, Ghiasi, Studer, Jacobs,
  and Goldstein]{shafahi2020AdvRobustTransfer}
Ali Shafahi, Parsa Saadatpanah, Chen Zhu, Amin Ghiasi, Christoph Studer,
  David~W. Jacobs, and Tom Goldstein.
\newblock Adversarially robust transfer learning.
\newblock In \emph{8th International Conference on Learning Representations},
  2020.

\bibitem[Shan et~al.(2020)Shan, Wenger, Zhang, Li, Zheng, and
  Zhao]{shan2020fawkes}
Shawn Shan, Emily Wenger, Jiayun Zhang, Huiying Li, Haitao Zheng, and Ben~Y
  Zhao.
\newblock Fawkes: Protecting privacy against unauthorized deep learning models.
\newblock In \emph{29th USENIX Security Symposium}, pp.\  1589--1604, 2020.

\bibitem[Sinha et~al.(2021)Sinha, Ayush, Song, Uzkent, Jin, and
  Ermon]{sinha2021nda}
Abhishek Sinha, Kumar Ayush, Jiaming Song, Burak Uzkent, Hongxia Jin, and
  Stefano Ermon.
\newblock Negative data augmentation.
\newblock In \emph{International Conference on Learning Representations}, 2021.

\bibitem[Song et~al.(2020)Song, Meng, and Ermon]{song2020ddim}
Jiaming Song, Chenlin Meng, and Stefano Ermon.
\newblock Denoising diffusion implicit models.
\newblock In \emph{International Conference on Learning Representations}, 2020.

\bibitem[Sun et~al.(2017)Sun, Shrivastava, Singh, and
  Gupta]{sun2017revisitData}
Chen Sun, Abhinav Shrivastava, Saurabh Singh, and Abhinav Gupta.
\newblock Revisiting unreasonable effectiveness of data in deep learning era.
\newblock In \emph{Proceedings of the IEEE international conference on computer
  vision}, pp.\  843--852, 2017.

\bibitem[Szegedy et~al.(2014)Szegedy, Zaremba, Sutskever, Bruna, Erhan,
  Goodfellow, and Fergus]{Szegedy2013IntrigueAdv}
Christian Szegedy, Wojciech Zaremba, Ilya Sutskever, Joan Bruna, Dumitru Erhan,
  Ian~J. Goodfellow, and Rob Fergus.
\newblock Intriguing properties of neural networks.
\newblock In \emph{2nd International Conference on Learning Representations},
  2014.

\bibitem[Taori et~al.(2020)Taori, Dave, Shankar, Carlini, Recht, and
  Schmidt]{taori2020NatShiftRobust}
Rohan Taori, Achal Dave, Vaishaal Shankar, Nicholas Carlini, Benjamin Recht,
  and Ludwig Schmidt.
\newblock Measuring robustness to natural distribution shifts in image
  classification.
\newblock In \emph{Advances in Neural Information Processing Systems},
  volume~33, pp.\  18583--18599, 2020.

\bibitem[Terj{\'e}k(2019)]{terjek2019wganALP}
D{\'a}vid Terj{\'e}k.
\newblock Adversarial lipschitz regularization.
\newblock In \emph{International Conference on Learning Representations}, 2019.

\bibitem[Tian et~al.(2020)Tian, Wang, Huang, Li, Dai, Yang, Wang, and
  Fink]{tian2020e2gan}
Yuan Tian, Qin Wang, Zhiwu Huang, Wen Li, Dengxin Dai, Minghao Yang, Jun Wang,
  and Olga Fink.
\newblock Off-policy reinforcement learning for efficient and effective gan
  architecture search.
\newblock In \emph{European Conference on Computer Vision}, pp.\  175--192.
  Springer, 2020.

\bibitem[Torralba \& Efros(2011)Torralba and Efros]{torralba2011databias}
Antonio Torralba and Alexei~A Efros.
\newblock Unbiased look at dataset bias.
\newblock In \emph{Conference on Computer Vision and Pattern Recognition
  (CVPR)}, pp.\  1521--1528. IEEE, 2011.

\bibitem[Tram{\`{e}}r et~al.(2020)Tram{\`{e}}r, Carlini, Brendel, and
  Madry]{Tramer2020AdaptiveAttack}
Florian Tram{\`{e}}r, Nicholas Carlini, Wieland Brendel, and Aleksander Madry.
\newblock On adaptive attacks to adversarial example defenses.
\newblock In \emph{33 Annual Conference on Neural Information Processing
  Systems}, 2020.

\bibitem[Uesato et~al.(2019)Uesato, Alayrac, Huang, Fawzi, Stanforth, and
  Kohli]{alayrac2019unsupadv}
Jonathan Uesato, Jean-Baptiste Alayrac, Po-Sen Huang, Alhussein Fawzi, Robert
  Stanforth, and Pushmeet Kohli.
\newblock Are labels required for improving adversarial robustness?
\newblock In \emph{Advances in Neural Information Processing Systems},
  volume~32, 2019.

\bibitem[Wang et~al.(2019)Wang, Xie, Li, Fonseca, and Tian]{wang2019lanet}
Linnan Wang, Saining Xie, Teng Li, Rodrigo Fonseca, and Yuandong Tian.
\newblock Sample-efficient neural architecture search by learning action space.
\newblock \emph{arXiv preprint arXiv:1906.06832}, 2019.

\bibitem[Wong \& Kolter(2018)Wong and Kolter]{wong2018provable}
Eric Wong and Zico Kolter.
\newblock Provable defenses against adversarial examples via the convex outer
  adversarial polytope.
\newblock In \emph{International Conference on Machine Learning}, pp.\
  5286--5295. PMLR, 2018.

\bibitem[Wong et~al.(2018)Wong, Schmidt, Metzen, and
  Kolter]{Wong2018ScalingProve}
Eric Wong, Frank~R. Schmidt, Jan~Hendrik Metzen, and J.~Zico Kolter.
\newblock Scaling provable adversarial defenses.
\newblock In \emph{31 Annual Conference on Neural Information Processing
  Systems}, pp.\  8410--8419, 2018.

\bibitem[Wu et~al.(2020{\natexlab{a}})Wu, Xia, and
  Wang]{wu2020AdvWeightPerturb}
Dongxian Wu, Shu-Tao Xia, and Yisen Wang.
\newblock Adversarial weight perturbation helps robust generalization.
\newblock \emph{Advances in Neural Information Processing Systems}, 33,
  2020{\natexlab{a}}.

\bibitem[Wu et~al.(2020{\natexlab{b}})Wu, Xia, and Wang]{wu2020adversarial}
Dongxian Wu, Shu-Tao Xia, and Yisen Wang.
\newblock Adversarial weight perturbation helps robust generalization.
\newblock volume~33, 2020{\natexlab{b}}.

\bibitem[Xu et~al.(2018)Xu, Evans, and Qi]{Xu2018featsqueeze}
Weilin Xu, David Evans, and Yanjun Qi.
\newblock Feature squeezing: Detecting adversarial examples in deep neural
  networks.
\newblock In \emph{25th Annual Network and Distributed System Security
  Symposium}, 2018.

\bibitem[Zagoruyko \& Komodakis(2016)Zagoruyko and Komodakis]{zagoruyko2016wrn}
Sergey Zagoruyko and Nikos Komodakis.
\newblock {Wide Residual Networks}.
\newblock In \emph{{British Machine Vision Conference}}, 2016.

\bibitem[Zhai et~al.(2019)Zhai, Cai, He, Dan, He, Hopcroft, and
  Wang]{zhai2019advJustmoreData}
Runtian Zhai, Tianle Cai, Di~He, Chen Dan, Kun He, John Hopcroft, and Liwei
  Wang.
\newblock Adversarially robust generalization just requires more unlabeled
  data.
\newblock \emph{arXiv preprint arXiv:1906.00555}, 2019.

\bibitem[Zhang et~al.(2019)Zhang, Yu, Jiao, Xing, El~Ghaoui, and
  Jordan]{zhang2019tradeoff}
Hongyang Zhang, Yaodong Yu, Jiantao Jiao, Eric Xing, Laurent El~Ghaoui, and
  Michael Jordan.
\newblock Theoretically principled trade-off between robustness and accuracy.
\newblock In \emph{International Conference on Machine Learning}, pp.\
  7472--7482, 2019.

\bibitem[Zhang et~al.(2020)Zhang, Chen, Xiao, Gowal, Stanforth, Li, Boning, and
  Hsieh]{Zhang2020crownIbp}
Huan Zhang, Hongge Chen, Chaowei Xiao, Sven Gowal, Robert Stanforth, Bo~Li,
  Duane~S. Boning, and Cho{-}Jui Hsieh.
\newblock Towards stable and efficient training of verifiably robust neural
  networks.
\newblock In \emph{8th International Conference on Learning Representations},
  2020.

\bibitem[Zhao et~al.(2020)Zhao, Liu, Lin, Zhu, and Han]{zhao2020diffAugGAN}
Shengyu Zhao, Zhijian Liu, Ji~Lin, Jun-Yan Zhu, and Song Han.
\newblock Differentiable augmentation for data-efficient gan training.
\newblock \emph{Advances in Neural Information Processing Systems}, 33, 2020.

\bibitem[{Zhao} et~al.(2020){Zhao}, {Zhou}, {Wang}, {Cai}, {Lun Lam}, and
  {Xu}]{zhao2020SplitNet}
Shuai {Zhao}, Liguang {Zhou}, Wenxiao {Wang}, Deng {Cai}, Tin {Lun Lam}, and
  Yangsheng {Xu}.
\newblock {Towards Better Accuracy-efficiency Trade-offs: Divide and
  Co-training}.
\newblock \emph{arXiv preprint arXiv:2011.14660}, 2020.

\bibitem[Zheng et~al.(2016)Zheng, Song, Leung, and
  Goodfellow]{zheng2016stabilityTrain}
Stephan Zheng, Yang Song, Thomas Leung, and Ian Goodfellow.
\newblock Improving the robustness of deep neural networks via stability
  training.
\newblock In \emph{Proceedings of the IEEE conference on computer vision and
  pattern recognition}, pp.\  4480--4488, 2016.

\end{thebibliography}
\bibliographystyle{iclr2022_conference}

\newpage
\appendix
\section{Theoretical results} \label{app: theory}
\newcommand{\red}[1]{{\color{red}{#1}}}
In this Section, we first provide the proofs of Theorem 1 and 2. Then we provide some additional results showing the tightness of our Theorem 1 and also the effect of combining real and proxy data on our Theorem. Finally, we provide some experimental result that validate our theory.
\subsection{Proof of Theorem 1}
\paragraph{Theorem 1:}
\textit{Let $D$ and $\tilde{D}$ be two labeled distributions supported on $\cX\times \cY$ with identical label distributions, i.e., $\forall y^* \in \cY, \Pr_{(x,y)\gets D}[y=y^*] =  \Pr_{(x,y)\gets \tilde{D}}[y=y^*]$. Then for any classifier $h:\cX\to \cY$
$$|\Rob_d(h,\tilde{D}) - \Rob_d(h,D)| \leq  \cwd_d(D,\tilde{D}).$$
}

\begin{proof}[Sketch of the proof]
We first provide an informal sketch of the proof and then formalize the steps after that. Consider $D'$ to be the distribution that is the outcome of the following process: First sample $(x,y)$ from $D$, then find the closest $x'$ such that $h(x')\neq y$ and output $(x',y)$\footnote{Here, we assume that the closest point $x'$ exists. Otherwise, We can set $x'$ so that the distance is arbitrarily close to the infimum and the proof follows.}. Also consider a similar distribution $\tilde{D}'$ corresponding to $\tilde{D}$. We now prove a Lemma that shows the conditional Wasserestein distance between $D$ and $D'$ is equal to $\Rob_d(h,D)$. 
\newbluetext{
\begin{lemma}\label{lem:RobEqCWD} We have
$\Rob_d(h,D)=\cwd(D,D')$ and $\Rob_d(h,\tilde{D})=\cwd(\tilde{D},\tilde{D}')$.
\end{lemma}
\begin{proof}
Let $J_y$ be the optimal transport between $D|y$ and $D'|y$. Also let $J'_y$ be the joint distribution $(x,x')$ that is obtained by first sampling $x$ from $D\mid y$ and then setting $x'=\argmin_{h(x')\neq y} d(x,x')$.
The marginals of $J'_y$ are equal to $D\mid y$ and $D'\mid y$. Hence, $J'_y$ is a valid transport between $D\mid y$ and $D'\mid y$. Also, define a potentially randomized perturbation algorithm $A^{J_y}$ that given an input $x$ samples $(x,x') \gets J_y \mid x$, conditioned on $x$ and outputs $x'$. 
We have
\begin{align*}\Rob_d(h,D)&=\Ex_{(x,y)\gets D}[\inf_{h(x')\neq y} d(x,x')]\\
&=\Ex_{(\cdot, y)\gets D}[ \Ex_{x\gets D\mid y}[\inf_{h(x')\neq y} d(x,x')]]=\Ex_{(\cdot, y)\gets D}[ \Ex_{(x,x')\gets J'_y}[d(x,x')]].
\end{align*}
On the other hand, because for all $x$ we have $h(A^{J_y}(x))\neq y$ then $d(x,A^{J_y}(x))\geq\inf_{y\neq h(x')} d(x,x')$. Therefore we have
$$\Rob_d(h,D)=\Ex_{(x,y)\gets D}[\inf_{h(x')\neq y} d(x,x')] \leq \Ex_{(x,y)\gets D}[d(x,A^{J_y}(x))]=\Ex_{(\cdot,y)\gets D}\Ex_{(x,x')\gets J_y}[d(x,x')]$$
Therefore we have
\begin{equation}\label{eq:00001}\Rob_d(h,D)=\Ex_{(\cdot, y)\gets D}[ \Ex_{(x,x)'\gets J'_y}[d(x,x')]] \leq \Ex_{(\cdot,y)\gets D}\Ex_{(x,x')\gets J_y}[d(x,x')].\end{equation}
On the other hand, because $J_y$ is the optimal transport, we have
\begin{equation}\label{eq:00002}\cwd_d(D,D')=\Ex_{(\cdot,y)\gets D}\Ex_{(x,x')\gets J_y}[d(x,x')] \leq \Ex_{(\cdot,y)\gets D}\Ex_{(x,x')\gets J'_y}[d(x,x')].\end{equation}

Now combining Equations \ref{eq:00001} and \ref{eq:00002}, we conclude that
$$\cwd_d(D,D')=\Ex_{(\cdot,y)\gets D}\Ex_{(x,x')\gets J_y}[d(x,x')] = \Ex_{(\cdot,y)\gets D}\Ex_{(x,x')\gets J'_y}[d(x,x')]=\Rob_d(h,D).$$
Similarly, we can also prove that $\cwd_d(\tilde{D},\tilde{D}')=\Rob_d(h,\tilde{D})$.
\end{proof}
}

By the way the distributions $D'$ and $\tilde{D'}$ are defined we have
\begin{equation}\label{ineq:proof2}
\cwd(D,D') \leq \cwd(D,\tilde{D'}) \text{~~~and~~~} \cwd(\tilde{D},\tilde{D}') \leq \cwd(\tilde{D},D').
\end{equation}

Roughly, the reason behind this is that all examples $(x',y)$ sampled from $\tilde{D'}$ could be seen as an adversarial example for all elements of $D$ with the label $y$. And we know that $D'$ consists of optimal adversarial examples for $D$, therefore, the optimal transport between $D$ and $D'$ should be smaller than the optimal transport between $D$ and $\tilde{D'}$.
Also, by triangle inequality for Wasserestein distance we have, 
\begin{equation}\label{ineq:proof1}
    \cwd(\tilde{D},D') \leq \cwd(\tilde{D},D) + \cwd(D,D').
\end{equation}

Now using Lemma \ref{lem:RobEqCWD} and Equations \ref{ineq:proof1} and \ref{ineq:proof2} we have
\begin{equation}\label{ineq:proof3}
\Rob_d(h,\tilde{D})=\cwd(\tilde{D},\tilde{D}')  \leq \cwd(\tilde{D},D') \leq \cwd(\tilde{D},D) + \cwd(D,D')=\cwd(\tilde{D},D) + \Rob_d(h,D).
\end{equation}
With a similar argument, because of symmetry of $D$ and $\tilde{D}$, we can also  prove
\begin{equation}\label{ineq:proof4}\Rob_d(h,D) \leq \cwd(\tilde{D},D) + \Rob_d(h,\tilde{D}).
\end{equation}

Combining inequalities \ref{ineq:proof3} and \ref{ineq:proof4} we get
\begin{equation}\label{ineq:proof4}-\cwd(\tilde{D},D)\leq \Rob_d(h,D) -  \Rob_d(h,\tilde{D})\leq \cwd(\tilde{D},D)
\end{equation}
which finishes the proof.
\end{proof}
\begin{proof}[Full proof]
The following is a succinct formalization of the proof steps mentioned above,
let $J^*_y=\inf_{J\in \mathcal{J}(D\mid y,\tilde{D}\mid y)}$ be the optimal transport between the conditional distributions $D\mid y$ and $\tilde{D}\mid y$. We have
\begin{align*}
    \Rob_d(h,D)&=\Ex_{(.,y)\gets D}\left[\Ex_{x\gets D\mid y}[\inf_{h(x')\neq y } d(x',x)]\right]\\
    &=\Ex_{(.,y)\gets D}\left[\Ex_{(x,x'')\gets J^*_y}[\inf_{h(x')\neq y } d(x',x)]\right]\\
    &\leq \Ex_{(.,y)\gets D}\left[\Ex_{(x,x'')\gets J^*_y}\left[\inf_{h(x')\neq y } d(x'',x') + d(x'',x)\right]\right]\\
    &=\Ex_{(.,y)\gets D}\left[\Ex_{(x,x'')\gets J^*_y}\left[\inf_{h(x')\neq y } d(x'',x')\right]\right]+\Ex_{(.,y)\gets D}\left[\Ex_{(x,x'')\gets J^*_y}[ d(x'',x)]\right]\\
    &=\Ex_{(x'',y)\gets \tilde{D}}[\inf_{h(x')\neq y } d(x'',x')] + \cwd_d(D,\tilde{D})\\
    &=\Rob_d(h,\tilde{D}) + \cwd_d(D,\tilde{D}).
\end{align*}
With a similar argument we get $\Rob_d(h,\tilde{D}) \leq \Rob_d(h,D) + \cwd_d(D,\tilde{D})$ and the proof is complete.
\end{proof}

\subsection{Proof of Theorem 2}
\paragraph{Theorem 2.}
\textit{
For any two distributions $\tilde{D}$ and $D$ with equal class probabilities we have $\cwd(D,\tilde{D})\geq 4\aurpcsf(D,\tilde{D})$.  Moreover, if for all labels $y$ , $(D\mid y)$ and $(\tilde{D}\mid y)$ are two concentric uniform $l_p$ spheres, then we have
$\cwd(D,\tilde{D}) = 4\aurpcsf(D,\tilde{D})$.\footnote{The statement of theorem in the main body has a typographical error that is fixed here.}
}
\begin{proof}
We start by proving a lemma.
\begin{lemma}\label{lem:totalvar}
Let $J \in \cJ(D,\tilde{D})$ be an arbitrary transport between $D$ and $\tilde{D}$. Then, for any discriminator $\sigma$ we have,
$$\Racc_{\epsilon/2}(\sigma,D,\tilde{D}) \leq \frac{1}{2} +\frac{\Pr_{(x,x')\gets J}\big[d(x,x')> \epsilon\big]}{2}.$$
\end{lemma}
\begin{proof}
Consider an adversary $A^\epsilon_J$ that on input $x$ sampled from $D$, samples a pair $(x,x')\gets J \mid J[1] = x$ from the transport and return a center point $x''$ between $x$ and $x'$ such that $d(x,x'')=d(x',x'') = d(x,x')/2$ if $d(x,x')\leq \epsilon$. Otherwise, it returns $x$. Also, on a in input $x'$ sampled from $\tilde{D}$ it samples a pair $(x,x')\gets J \mid J[2] = x'$ from the transport and return a center point $x''$ between $x$ and $x'$ such that $d(x,x'')=d(x',x'') = d(x,x')/2$ if $d(x,x')\leq \epsilon$. Otherwise, it returns $x'$.

Denote $D_A\equiv A^\epsilon_J\#D$ to be the push-forward of $D$ under $A_J$. And also denote $\tilde{D}_A\equiv A^\epsilon_J\#\tilde{D}$. Now we have, $\delta(D_A,\tilde{D}_A) \leq \Pr_{(x,x')\gets J}[d(x,x')>\epsilon]$ where $\delta$ is the total variational distance. Therefore, no discriminator $\sigma$ can distinguish $D'$ from $\tilde{D}$ with accuracy more than $\frac{1 +\Pr_{(x,x')\gets J}[d(x,x')>\epsilon]}{2}$.
\end{proof}
Now, let $J^* = \sum_{y=1}^l \alpha_y J^*_y$ be a transport between $D$ and $\tilde{D}$ where $J^*_y$ is the optimal transport between $\tilde{D}\mid y$ and $D \mid y$ and $\alpha_y$ is the probability of class $y$ for $D$ and $\tilde{D}$. using Lemma \ref{lem:totalvar} we have
$$\Racc^*_\epsilon(D,\tilde{D}) -\frac{1}{2}\leq \frac{\Pr_{(x,x')\gets J^*}[d(x,x')>2\epsilon]}{2}.$$
Therefore we have

\begin{equation}\label{eq0000}
\int_0^\infty [\Racc^*_\epsilon(D,\tilde{D}) -\frac{1}{2}] d\epsilon \leq \int_0^\infty \frac{\Pr_{(x,x')\gets J^*}[d(x,x')>2\epsilon]}{2} d\epsilon =\int_0^\infty \frac{\Pr_{(x,x')\gets J^*}[d(x,x')>\epsilon]}{4} d\epsilon .
\end{equation}
Now, by integration by parts we have 
\begin{align*}\int_0^\infty \Pr_{(x,x')\gets J^*}[d(x,x')> \epsilon] d\epsilon  &= \epsilon \Pr_{(x,x')\gets J^*}[d(x,x')> \epsilon]\Big|_0^\infty + \int_0^\infty \Pr_{(x,x')\gets J^*}[d(x,x')=\epsilon]\epsilon   d\epsilon\\
&=0 + \Ex_{(x,x')\gets J^*}[d(x,x')]\\
&=\cwd(D,\tilde{D}).
\end{align*}
Therefore, we have
$$\aurpcsf(D,\tilde{D}) = \int_0^\infty [\Racc^*_\epsilon(D,\tilde{D}) -\frac{1}{2}] d\epsilon \leq \frac{\cwd(D,\tilde{D})}{4}.$$
Now to prove the second part of theorem about the case of concentric spheres we start by following lemma. 
\begin{lemma}\label{lem:totalvarasphere}
Let $D_y=D\mid y$ and $\tilde{D}_y=\tilde{D}\mid y$. Assume $\tilde{D}_y$ and $D_y$ are two concentric spheres according to the some $l_p$ norm. Also let $J^* \in \cJ(D_y,\tilde{D}_y)$ be the optimal transport between $D_y$ and $\tilde{D}_y$ with respect to the same norm.  Then we have,
$$\Racc^*_{\epsilon/2}(D,\tilde{D}) = \frac{1}{2} +\frac{\Pr_{(x,x')\gets J^*}\big[d(x,x')> \epsilon\big]}{2}.$$
\end{lemma}
\begin{proof}
let $D$ be the uniform sphere centered at point $c$ and with radius $r$ and let $\tilde{D}$ be the uniform sphere centered at point $c$ and with radios $\tilde{r}>r$. Observe that the optimal robust distinguisher for any $\epsilon<\frac{r+\tilde{r}}{2}$ is the one that assigns $0$ to $x$ if $d(c,x)<\frac{\tilde{r}-r}{2}$ and assigns $1$ otherwise and for any $\epsilon< \frac{\tilde{r}-r}{2}$ we have
\begin{equation}\label{eq:0001}
    \Racc^*_{\epsilon/2}(D,\tilde{D})=1.
\end{equation}
On the other hand, for any $\epsilon\geq \frac{\tilde{r}-r}{2}$ we have 
\begin{equation}\label{eq:0002}
    \Racc^*_{\epsilon/2}(D,\tilde{D})=0.5.
\end{equation}

Now consider optimal transport $J^*$ between $D$ and $\tilde{D}$. Observe that this transport is equivalent to the following distribution: $$J^*\equiv (D, (D-c)\cdot \tilde{r}/r)$$
Therefore,  we have 
$\Pr_{(x,x')\gets J^*}[d(x,x')=\tilde{r}-r]=1$ 
Which implies for any $\epsilon<\tilde{r}-r$ we have
\begin{equation}\label{eq:0003}
\Pr_{(x,x')\gets J^*}[d(x,x')> \epsilon]=1
\end{equation}
and for any $\epsilon\geq\tilde{r}-r$
\begin{equation}\label{eq:0004}
\Pr_{(x,x')\gets J^*}[d(x,x')> \epsilon]=0.
\end{equation}
Putting Equations \eqref{eq:0001}, \eqref{eq:0002}, \eqref{eq:0003}, and \eqref{eq:0004} together, we finish the proof of Lemma. 
\end{proof}
Having Lemma \ref{lem:totalvarasphere}, we can follow the same steps as before except that instead of Inequality \eqref{eq0000} we have an equality of the following form.

\begin{align*}\label{eq0000eq}
\int_0^\infty [\Racc^*_\epsilon(D,\tilde{D}) -\frac{1}{2}] d\epsilon &= \int_0^\infty \frac{\Pr_{(x,x')\gets J^*}[d(x,x')>2\epsilon]}{2} d\epsilon\\ &=\int_0^\infty \frac{\Pr_{(x,x')\gets J^*}[d(x,x')>\epsilon]}{4} d\epsilon\\
&=\frac{\cwd(D,\tilde{D})}{4}.
\end{align*}
\end{proof}
\subsection{Effect of Combining Proxy and Real Data on Theorem 1}
A natural question is what happens when we combine the original distribution with the proxy distribution. For example, one might have access to a generative model but they want to combine the samples from the generative model with some samples from the original distribution and train a robust classifier on the aggregated dataset. The following corollary answers this question. 
\begin{theorem}\label{thm:tight}
Let $D$ and $\tilde{D}$ be two labeled distributions supported on $X\times Y$ with identical label distributions and let $\bar{D} = p\cdot D + (1-p)\cdot \tilde{D}$ be the weighted mixture of $D$ and $\tilde{D}$. Then for any classifier $h:X\to Y$
$$|\Rob_d(h,\bar{D}) - \Rob_d(h,D)| \leq  (1-p)\cdot \cwd_d(D,\tilde{D}).$$
\end{theorem}
Note that the value of $p$ is usually very small as the number of data from proxy distribution is usually much higher than the original distribution. This shows that including (or not including) the data from original distribution should not have a large effect on the obtained bound on distribution-shift penalty. 

\begin{proof}[Proof of Theorem \ref{thm:tight}]

We just need to show that $\cwd_d(D,\bar{D})\leq (1-p)\cdot \cwd_d(D,\tilde{D})).$ Note that since the label distributions are equal, we have $$\bar{D}\mid y  \equiv p\cdot D\mid y + (1-p)\cdot\tilde{D}\mid y.$$ Now let $J_y$ be the optimal transport between $D\mid y$ and $\tilde{D}\mid y$. Now construct a joint distribution $J'_y \equiv (1-p) \cdot J + p\cdot (x,x)_{x\gets D|y}$. Notice that $J'_y$ is a joint distribution with marginals equal to $D$ and $\bar{D}$. Therefor $J'_y$ is a transport between $D$ and $\bar{D}$ and we can calculate its cost. We have

$$\Ex_{(x,x')\gets J'_y}[d(x,x')] = (1-p)\cdot \Ex_{(x,x')\gets J_y}[d(x,x')] + \Ex_{x\gets D|y}[d(x,x)]=(1-p)\cdot \Ex_{(x,x')\gets J_y}[d(x,x')].$$
Therefore, we have
$$\cwd(D,\bar{D})\leq \Ex_{(\cdot,y)\gets D}[ \Ex_{(x,x')\gets J'_y}[d(x,x')]] = (1-p)\cdot\Ex_{(\cdot,y)\gets D}[ \Ex_{(x,x')\gets J_y}[d(x,x')]] =  (1-p)\cdot \cwd(D,\tilde{D}).$$
\end{proof}

\subsection{Tightness of Theorem 1}
Here, we show a theorem that shows our bound on distribution-shift penalty is tight. The following theorem shows that one cannot obtain a bound on the distribution-shift penalty for a specific classifier that is \emph{always} better than our bound.

\begin{theorem}[Tightness of Theorem 1] \label{th: robustboundtight}
For any distribution $D$ supported on $X\times Y$, any classifier $h$, any homogeneous distance $d$ and any $\epsilon\leq \Rob_d(h,D)$, there is a labeled distribution $\tilde{D}$ such that 
$$\Rob_d(h,D) -\Rob_d(h,\tilde{D}) = \cwd(D,\tilde{D}) =\epsilon.$$
\end{theorem}
\begin{proof}
for $\alpha\in[0,1]$ let $\tilde{D}_{\alpha}$ be the distribution of the following process: First sample $(x,y)$ from $D$, then find the closest $x'$ such that $h(x')\neq y$ and output $(x+ \alpha(x'-x),y)$. By definition, the conditional Wasserestein distance between $D$ and $\tilde{D}_1$ is equal to $\Rob_d(h,D)$. We also have $\cwd(D,\tilde{D}_\alpha) = \alpha\cdot \cwd(D,\tilde{D}_1)$. 

Observe that for any classifier we have $\Rob_d(h,\tilde{D}_\alpha)\leq (1-\alpha) \Rob_d(h,D)$ because if $(x',y)$ is an adversarial example for $(x,y)$, then $x'$ is also an adversarial example for $(x+ \alpha(x'-x), y)$ with distance $(1-\alpha) d(x,x')$. On the other hand we have $\Rob_d(h,\tilde{D}_\alpha)\geq (1-\alpha)\Rob_d(h,D)$ because any adversarial example for $(x+\alpha(x'-x),y)$ with distance $r$ is also an adversarial example for $x$ with distance at most $r + \alpha d(x'-x)$ and since $x'$ is the optimal adversarial example for $x$ then $r$ must be at least $\alpha(x'-x)$. Therefore, we have $\Rob_d(h,\tilde{D}_\alpha) = (1-\alpha)\Rob_d(h,D).$ Putting everything together and setting $\alpha=\epsilon/\Rob_d(h,D)$we have

$$\Rob_d(h,D) - \Rob_d(h,\tilde{D}_\alpha) = \alpha \Rob_d(h,D) = \alpha \cwd(D,\tilde{D}_1) =\cwd(D,\tilde{D}_\alpha)=\epsilon.$$
\end{proof}
Note that Theorem \ref{th: robustboundtight} only shows the tightness of Theorem 1 for a specific classifier. But there might exist a learning algorithm $L$ that incurs a much better bound in the expectation. Namely, there might exist $L$ such that for any two distributions $D$ and $\tilde{D}$ we have  

$$\big|\Ex_{\substack{S\gets \tilde{D}^n\\ h\gets L(S)}}[\Rob_d(h,D)-\Rob_d(h,\tilde{D})]\big| \ll \cwd(D,\tilde{D}).$$
We leave finding such an algorithm as an open question.

\subsection{Experimental validation of main theorem} \label{app: theory_validation}
\begin{figure}
    \centering
    \includegraphics[width=0.5\linewidth]{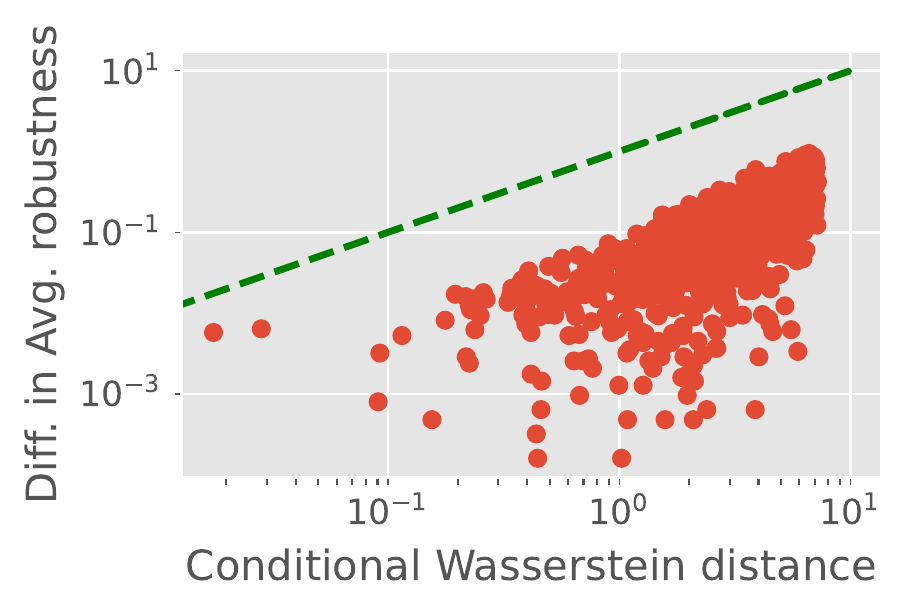}
    \caption{Validating the upper bound from Theorem 1. The green line is the upper bound calculated by Wasserstein-2. Note that the Wasserstein-1 (The bound of Theorem 1) is a tighter upper-bound but it does not have a closed form for normal distributions.}
    \label{fig: valid_theorem}
\end{figure}
In theorem 1 we prove that transfer of adversarial robustness between two distributions is upper bounded by their conditional Wasserstein distance. Now we provide empirical validation of this result. We consider a $10$-class classification problem where each class is modeled using a $128$-dimensional multivariate normal distribution. In contrast to real-world datasets, where the underlying distribution of the data is unknown, we can efficiently and exactly calculate Wasserstein distance on this synthetic dataset. We robustly train a four-layer convolutional network on it. To construct a proxy distribution, we perturb both mean and covariance parameters across each class. On each proxy distribution, we measure the average robustness of the pre-trained model and its conditional Wasserstein distance from the \orgdist distribution. We use Wasserstein-2 distance\footnote{The bound in our theorem is based on $W_1$. Since we know $W_1 < W_2$, using $W_2$ provides a loose upper bound.}, which has a closed-form expression for normal distributions~\citep{givens1984classWasserstein}.
As shown in Figure~\ref{fig: valid_theorem}, the difference in average robustness is upped bounded by the conditional Wasserstein distance.

\section{Additional details on experimental setup} \label{app: setup}
We describe experimental details common across most experiments in this section. We discuss design choices pertaining to individual experiments in their respective sections.

\noindent \textbf{Training setup.} We use network architectures from the ResNet family, namely ResNet-18~\citep{he2016resnet} and variants of WideResNet~\citep{zagoruyko2016wrn}. We train each network using stochastic gradient descent and $0.1$ learning rate with cosine learning rate decay, weight decay of $5\times10^{-4}$, batch size $128$, and $200$ epochs. We use $1\times10^{-3}$ weight decay on CIFAR-100 dataset, since higher weight decay further reduces the generalization gap on this dataset. 
We work with five datasets, namely CIFAR-10, CIFAR-100, ImageNet ($64\times64$ size images), CelebA ($64\times64$), and AFHQ ($224\times224$). For CelebA~\citep{liu2015CelebA} dataset, we consider a four-class classification problem based on the attribute smile and male, i.e., smile/male, smile/not-male, not-smile/male, not-smile/not-male (total $200$K images). Animal Faces-High-Quality (AFHQ) dataset comprises $15$K images for three classes: cat, dog, and wild animals. Thus we consider the task of three-class classification on this dataset. To reduce the computational cost on ImageNet dataset, we use the free adversarial training~\citep{shafahi2019freeAdv} procedure. \newtext{We use the ResNet-18 model for CIFAR-10, ImageNet, CelebA, and AFHQ dataset. However, for CIFAR-10 and CIFAR-100 dataset, we also use much larger \wrn{34-10} network.}

\noindent \textbf{Robust training parameters.} 
We consider both $l_{\infty}$ and $l_2$ threat models. For CIFAR-10 and CIFAR-100 dataset, we use the commonly used perturbation budget ($\epsilon$) of $8/255$ and $127/255$ for $l_{\infty}$ and $l_2$ threat model, respectively. For other three dataset too, we choose commonly used perturbation budgets. We perform adversarial training using a 10-step projected gradient descent attack (PGD-10) and benchmark test set robustness with the much stronger AutoAttack\footnote{We don't report numbers with PGD attacks as AutoAttack already captures them while also making it easier to compare with other works~\citep{croce2020robustbench}.}~\citep{croce2020autoattack}. We use $\beta=6$, $\sigma=0.25$, $100$ samples for selection, and $10,000$ samples for estimation in randomized smoothing, as described in \citet{cohen2019certified}. 


\noindent \textbf{Generative models.} \newtext{We consider seven different generative models for the CIFAR-10 and two different networks for the CelebA and ImageNet dataset. In Table~\ref{tab: all_gans}. We provide the number of synthetic images generated from each model and whether the generated images are labelled or unlabeled.
} If the model is unconditional, i.e., only generates unlabeled images, we label the images using an additional classifier.
We use state-of-the-art LaNet~\citep{wang2019lanet} network to label images for the CIFAR-10 dataset. LaNet network is trained only on the training images of the CIFAR-10 dataset and achieves 99.0\% accuracy on its test set. \newtext{For CelebA dataset, we train a ResNet50 model to label synthetic images.} \newtext{For adaptive sampling of synthetic images for CIFAR-10 dataset, we first sample of set of $15$M\footnote{\bluetext{We use the $6$M randomly sampled images made available by \citet{nakkiran2021deepbootstrap} along with $9$M more sampled by us using improved sampling techniques~\citep{nichol2021improvedDdpm} from the \ddpm model}.} synthetic images from the DDPM model. Next we  select $10$M synthetic images with the lowest synthetic score from this set.
}



\begin{table}[!htb]
\centering
\caption{\newtext{\textbf{Generative models for each dataset.} In this table we list the different generative models used for each dataset and the number of synthetic images sampled form each model. We also indicate whether the generative models generate labeled images, i.e., class-conditioned. If not class-conditioned, the model only generated unlabeled synthetic images.}}
\label{tab: all_gans}
\renewcommand{\arraystretch}{1.3}
\resizebox{0.9\linewidth}{!}{
\begin{tabular}{ccccc} \toprule
Dataset & \begin{tabular}[c]{@{}c@{}}Number of training \\ images in dataset\end{tabular} & Generative model & \begin{tabular}[c]{@{}c@{}}Number of \\ synthetic images\end{tabular} & Class-conditioned \\ \midrule
\multirow{7}{*}{CIFAR-10} & \multirow{7}{*}{$50$K} & DDPM~\cite{ho2020denoisingdiffusion} & $10$M & $\times$ \\
 &  & StyleGAN~\citep{karras2020styleganAda} & $10$M & $\checkmark$ \\
 &  & WGAN-ALP~\citep{terjek2019wganALP} & $1$M & $\times$ \\
 &  & E2GAN~\citep{tian2020e2gan} & $1$M & $\times$ \\
 &  & DiffCrBigGAN~\citep{zhao2020diffAugGAN} & $1$M & $\checkmark$ \\
 &  & NDA~\citep{sinha2021nda} & $1$M & $\checkmark$ \\
 &  & DiffBigGAN~\citep{zhao2020diffAugGAN} & $1$M & $\checkmark$ \\ \midrule
\multirow{2}{*}{CelebA} & \multirow{2}{*}{$120$K} & StyleFormer~\citep{park2021styleformer} & $1$M & $\times$ \\
 &  & DDPM~\citep{song2020ddim} & 1M & $\times$ \\ \midrule
\multirow{2}{*}{ImageNet} & \multirow{2}{*}{$1.2$M} & BigGAN~\citep{brock2018bigGandeep} & $1$M & $\checkmark$ \\
 &  & DDPM~\citep{nichol2021improvedDdpm} & $400$K & $\checkmark$ \\ \midrule
CIFAR-100 & $50$K & DDPM~\citep{nichol2021improvedDdpm} & 1M & $\checkmark$ \\ \midrule
AFHQ & $15$K & StyleGAN~\citep{karras2020styleganAda} & $300$K & $\checkmark$ \\ \bottomrule
\end{tabular}}
\end{table}


\noindent \textbf{Evaluation metrics for generative models.} We use following existing baseline metrics to evaluate the quality of synthetic samples. 1) \emph{Fréchet Inception Distance (FID)}: It measures the Fréchet distance~\citep{dowson1982frechetDistance} between features of synthetic and real images extracted from an Inception-V3 network. 2) \emph{Inception score (IS)}: Unlike FID, Inception score only uses synthetic data with the goal to account for both fidelity and diversity of synthetic data. 3) \emph{Nearest neighbour distance (1-NN):} It computes the average distance of a synthetic image to the nearest real image in the pixel space. 4) \bluetext{\emph{\aurpc}: This is the metric we propose and it measures distance between synthetic and real data based on the success of a robust discriminator.} We compare performance of all three baselines with our proposed metric. 

\noindent \textbf{Computational cost.} In addition to robust training, sampling from generative models is another key contributor to the computational cost in our approach. We sample images from the \ddpm model using $250$ steps for both CIFAR-10 and ImageNet datasets. Using an RTX 4x2080Ti GPU cluster, it takes $23.8$ hours to sample one million images on the CIFAR-10 dataset. With the same setup, it takes $26.1$ hours to sample $100$K images for the $64\times64$ ImageNet dataset. On both models, we use the publicly available checkpoints of pretrained generative models\footnote{\url{https://github.com/openai/improved-diffusion}}. Note that both training generative models and sampling from them is a one-time cost. Once the pretrained checkpoints and synthetic images are made publicly available, they can be directly used in downstream tasks. We will make our code and synthetic data publicly available.


\begin{table}[!htb]
    \caption{Using synthetic data also improves clean and robust accuracy on AFHQ~\citep{choi2020starganv2} dataset. Baseline refers to adversarial training based on Madry et al.~\citep{madry2017towards}. \textit{Clean}/\textit{Auto} refers to clean/robust accuracy measured with AutoAttack. We use a ResNet-18 network and randomly sampled synthetic images.}
    \label{tab: afhq_table}
    \centering
    \renewcommand{\arraystretch}{1.4}
    \Large
    \resizebox{0.6\linewidth}{!}{
    \begin{tabular}{cccccccccccc} \toprule
         & \multicolumn{5}{c}{$\ell_\infty$} &  & \multicolumn{5}{c}{$\ell_2$} \\ \cmidrule{2-6} \cmidrule{8-12}
        $\epsilon$ & \multicolumn{2}{c}{$4/255$} &  & \multicolumn{2}{c}{$8/255$} &  & \multicolumn{2}{c}{$3.0$} &  & \multicolumn{2}{c}{$5.0$} \\ \midrule
         & \textit{Clean} & \textit{Auto} &  & \textit{Clean} & \textit{Auto} &  & \textit{Clean} & \textit{Auto} &  & \textit{Clean} & \textit{Auto} \\
        Baseline & $98.8$ & $93.3$ &  & $98.4$ & $84.3$ &  & $98.9$ & $93.7$ &  & $98.7$ & $88.8$ \\
        \ours & $99.1$ & $93.5$	&  & $98.8$ & $86.5$ &  & $99.0$ & $94.0$ &  & $98.8$ & $89.2$ \\
        $\Delta$ & $\boldsymbol{\texttt{+}0.3}$ & $\boldsymbol{\texttt{+}0.2}$	&  & $\boldsymbol{\texttt{+}0.4}$ & $\boldsymbol{\texttt{+}2.2}$ &  & $\boldsymbol{\texttt{+}0.1}$ & $\boldsymbol{\texttt{+}0.3}$ &  & $\boldsymbol{\texttt{+}0.1}$	& $\boldsymbol{\texttt{+}0.4}$ \\ \bottomrule
    \end{tabular}}
\end{table}



\section{Delving deeper into robust discrimination and \emph{\aurpc}} \label{app: aurpc}
We propose a new metric (\emph{\aurpc}), to rank different generative models in order of robustness transfer from their samples to real data. Now we provide experimental details on how we measure \emph{\aurpc} and its comparison with other baselines.

\noindent \textbf{Setup.} A critical component in our approach is training a robust binary discriminator to distinguish between real and synthetic data. We use a ResNet-18 network for this task and train it for $100$ epochs with a $0.1$ learning rate, $128$ batch size, and 1e-4 weight decay. We use a randomly sampled set of one million images from each generative model. We keep $10,000$ synthetic images from this set for validation and train on the rest of them. We specifically choose $10,000$ images as it's equal to the number of test images available in the CIFAR-10 test set, thus balancing both classes in our test set. We train for only $391$ steps per epoch, i.e., equivalent to a single pass through $50,000$ training images in the CIFAR-10 dataset. Thus over $100$ epochs, we effectively end up taking five passes through the one million synthetic images. Given the randomness in multiple steps, we aggregate results over three different runs of our experiments. 

\subsection{Why non-robust discriminators are not effective in measuring proximity?} \label{app: nonrobust_disc}
We argue that most synthetic samples can be easily distinguished from real data, irrespective of proximity, in absence of adversarial perturbations ($\epsilon=0$). This is likely because deep neural networks have very high expressive power. So a natural question is whether shallow networks, which have much lower capacity, are more suitable for this task. We test this hypothesis by using a two-layer (16 and 32 filters) and a four-layer (16, 16, 32, and 32 filters) convolutional neural network. We refer to them as CNN-2 and CNN-4, respectively. We find that even these shallow networks achieve more than $90$\% accuracy in distinguishing synthetic images from real images (Table~\ref{tab: shallowCNN}). The CNN-4 network itself is achieving more than $99$\% accuracy for four out of seven generative models. A much larger ResNet-18 network achieves near-perfect classification accuracy for most models. Such high success of even shallow networks uniformly across generative models shows that non-robust discrimination is not an effective measure of proximity. 

\begin{table}[]
    \centering
    \caption{Success of different network architectures in classifying synthetic images from real-world image on CIFAR-10 dataset.}
    \label{tab: shallowCNN}
    \resizebox{0.95\linewidth}{!}{
    \begin{tabular}{cccccccc}
        \toprule
        Model & \ddpm & \styleC & WGAN-ALP & E2GAN & DiffCrBigGAN & NDA &  DiffBigGAN \\ \midrule
        CNN-2 & $59.1$ &  $94.3$ & $99.1$  & $97.3$  & $92.2$ &  $95.4$ & $97.7$ \\
        CNN-4 & $97.0$ & $99.1$ & $99.9$ & $99.7$ & $94.6$ &  $99.8$ & $99.7$  \\
        ResNet-18 & $98.3$ & $99.9$ & $99.9$ & $99.9$ & $99.6$ & $99.9$ & $99.9$ \\ \bottomrule
    \end{tabular}}
\end{table}

\subsection{On effectiveness of \emph{\aurpc} in measuring proximity}   \label{app: robust_disc}
In this section, we delve deeper into the comparison of our proposed metric (\emph{\aurpc}) with other baselines. We judge the success of each metric by how successfully it predicts the transfer of robustness from synthetic to real samples on the CIFAR-10 data. (referred to as transferred robust accuracy). We report our results in Table~\ref{tab: arc_detailed}.

\noindent \textbf{Is FID in CIFAR-10 feature space effective?} When calculating FID, the recommended approach is to measure it in the feature space of a network trained on the ImageNet dataset~\citep{Heusel2017FID}. One may ask, whether the reason behind the failure is using ImageNet feature space for CIFAR-10 images. This might be true as it is well known that most image datasets have their specific bias~\citep{torralba2011databias}. To answer this question, we measure FID in the feature of a state-of-the-art LaNet~\citep{wang2019lanet} networks trained on the CIFAR-10 dataset. We refer to this metric as FID (CIFAR-10). Similar to FID in ImageNet feature space, FID in CIFAR-10 feature space also fails to predict robustness transfer. It follows a similar trend as the former, where it ranks models like DiffBigGAN and StyleGAN much higher than other models. 

\textbf{Generalizing to natural training.} If the objective is to maximize clean accuracy instead of robustness, one may ask whether \emph{\aurpc} will remain effective in selecting the generative model. We answer this question by measuring the success of \aurpc in predicting how much clean accuracy is achieved on the CIFAR-10 test set when we train a ResNet-18 network only on one million synthetic images without any adversarial perturbation. Note that \emph{\aurpc} is dependent only on samples, thus can be directly used for this task. We find that \aurpc correctly ranks \ddpm over \styleC model for transferred clean accuracy. Similarly, it also correctly rank WGAN-ALP and E2GAN model higher than other BigGAN models where the former achieves higher clean accuracy on the CIFAR-10 test set. However, there are small variations in transferred clean accuracy with some models, such as WGAN-ALP and E2GAN, which \emph{\aurpc} doesn't capture in its ranking. 

\begin{table}[!htb]
    \centering
    \caption{Testing how much each proxy distribution helps on CIFAR-10 dataset and whether FID/IS captures it. We train on 1M synthetic images and measure transferred robustness (which determines the rank) to cifar10 test set. UC and C refers to unconditional and conditional generative models, respectively.}
    \label{tab: arc_detailed}
    \resizebox{0.95\linewidth}{!}{
    \begin{tabular}{cccc||ccccc}
        \toprule
        Rank & Model & \begin{tabular}[c]{@{}c@{}}Transferred clean accuracy\\ (Natural training)  ($\uparrow$)\end{tabular} & \begin{tabular}[c]{@{}c@{}}Transferred robust accuracy\\ (Adversarial training)  ($\uparrow$)\end{tabular} & \begin{tabular}[c]{@{}c@{}}FID \\ (ImageNet)  ($\downarrow$)\end{tabular} & \begin{tabular}[c]{@{}c@{}}FID \\ (CIFAR-10)  ($\downarrow$)\end{tabular} & IS  ($\uparrow$) & \oneNN  ($\downarrow$) & \aurpc  ($\downarrow$) \\ \midrule
        $1$ & \ddpm (UC) & $94.8$ & $53.1$ & $3.17$ & $0.90$ &  $9.46$ &  $9.34$ &  $0.06$ \\
        $2$ & \styleC (C) & $91.2$ & $45.0$ & $2.92$ & $0.55$ & $10.24$ &  $9.42$ & $0.32$ \\
        $3$ & WGAN-ALP (UC) & $92.3$ & $43.5$ & $12.96$ & $5.84$ & $8.34$ & $10.10$  & $1.09$ \\ 
        $4$ & E2GAN (UC) & $92.7$ & $39.6$ & $11.26$ & $4.52$ & $8.51$ & $8.96$  & $1.19$   \\
        $5$ & DiffCrBigGAN (C) & $87.4$ & $33.7$ & $4.30$ & $0.96$ &  $9.17$&  $9.84$  & $1.30$  \\ 
        $6$ & NDA (C) &  $84.9$ & $33.4$ & $12.61$ & $0.78$ & $8.47$ & $9.72$  & $1.43$ \\
        $7$ & DiffBigGAN (C)  & $86.7$ & $32.4$ & $4.61$ & $0.64$ & $9.16$ &  $9.73$ & $1.55$ \\ \bottomrule
    \end{tabular}}
\end{table}

\subsection{Adaptive sampling of synthetic data.} \label{appsubsec: adaptive}

\textbf{Finding synthetic samples which leads to highest synthetic-to-real robustness transfer.} \label{sec: rejection_sampling}
We use proximity of a synthetic sample to the \orgdist distribution, which we measure using synthetic score from our trained robust discriminators (using $0.25/255$ size $l_\infty$ perturbations), as a metric to judge transfer of performance from the \approxdist distribution to \orgdist distribution. 
\begin{wrapfigure}{r}{0.4\textwidth}
    \vspace{-15pt}
    \centering
    \includegraphics[width=0.8\linewidth]{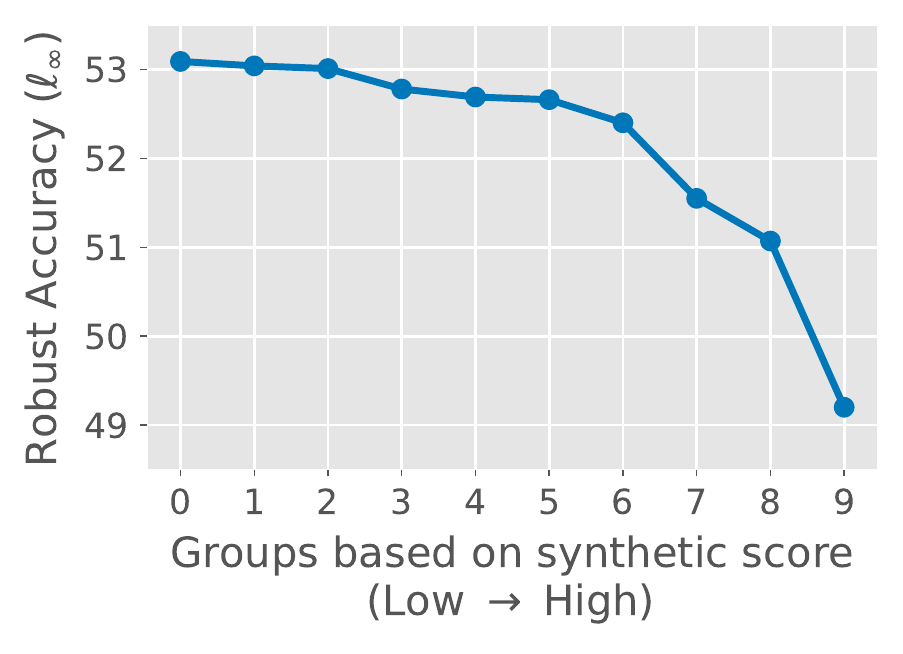}
    \vspace{-10pt}
    \caption{\textbf{Effectiveness of synthetic score.} \bluetext{We sort $6$M \ddpm images using their synthetic score and divide them into ten equal-size groups. We find that groups with lower score achieve higher transferred robust accuracy ($\ell_{\infty}$) on the CIFAR-10 test set.
    \vikash{Use auto-attack numbers. Plot a lower bar, one without any synthetic data. Clearly highlight what legend lp norm vs what its being trained against.}}}
    \label{fig: transfer_splits}
    \vspace{-20pt}
\end{wrapfigure}
We sort the \synthetic score for each of $6$M synthetic images from the \ddpm model and combine them into ten equal-size groups. Next, we adversarially train a ResNet-18 network on images from each group, along with CIFAR-10 training images, and measure robustness achieved on the CIFAR-10 test set  (Figure~\ref{fig: transfer_splits}). Our results validate the effectiveness of \synthetic scores, where groups with the lowest \synthetic score achieve highest robust test accuracy. The difference in the robust test accuracy is up to $4$\% between the group with the lowest score and the one with the highest score. 

\textbf{Improvement in performance with adaptive sampling.}
\newtext{We select of set of $10$M synthetic images (from a set of $15$M) with the lowest synthetic score. We compare the success of these adaptively selected images with a set of $10$M randomly selected images. We present our results in Table~\ref{tab: sota_app}. It shows that adaptively sampling further improves the performance of our framework ($\ours$) across threat models ($\ell_{\infty}$ and $\ell_2$) and network architectures.}

\begin{table}[!tb]
    \centering
    \caption{\textbf{Further improvement in adversarial robustness using adaptive sampling.} Experimental results with adversarial training on the CIFAR-10 dataset for both $\ell_{\infty}$ and $\ell_2$ threat model. Using additional synthetic data brings a large gain in adversarial robustness across network architectures and threat models. \bluetext{We also show further improvements brought in by adaptive sampling, in comparison to random sampling, of synthetic images.} \textit{Clean}/\textit{Auto} refers to clean/robust accuracy measured with AutoAttack.}
    \label{tab: sota_app}
    \vspace{-5pt}

    \begin{subtable}[h]{0.49\textwidth}
        \centering
        \caption{$\ell_{\infty}$ threat model.}
        \renewcommand{\arraystretch}{1.2}
        \resizebox{0.99\linewidth}{!}{\begin{tabular}{ccccc}
        \toprule
        Method & Architecture & Parameters (M) & \textit{Clean} & \textit{Auto} \\ \midrule
        \citet{zhang2019tradeoff} & ResNet-18 & $11.2$ & $82.0$ & $48.7$ \\ 
        \ours (random) & ResNet-18 & $11.2$ & $84.4$ & $55.6$ \\
        \ours (adaptive) & ResNet-18 & $11.2$ & $\boldsymbol{84.6}$ & $\boldsymbol{55.7}$ \\ \midrule
        \citet{rice2020overfitadv} & \wrn{34-20} & $184.5$ & $85.3$ & $53.4$ \\
        \citet{gowal2020uncoveringadv} & \wrn{70-16} & $266.8$ & $85.3$ & $57.2$
        \\ \midrule
        \citet{zhang2019tradeoff} & \wrn{34-10} & $46.2$ & $84.9$ & $53.1$ \\
        \ours (random) & \wrn{34-10} & $46.2$ & $86.7$ & $60.3$ \\
        \ours (adaptive) & \wrn{34-10} & $46.2$ & $\boldsymbol{87.0}$ & $\boldsymbol{60.6}$ \\
        \bottomrule
        \end{tabular}}
    \end{subtable}\hfill
        \begin{subtable}[h]{0.49\textwidth}
        \centering
        \caption{$\ell_2$ threat model.}
        \renewcommand{\arraystretch}{1.2}
        \resizebox{0.99\linewidth}{!}{\begin{tabular}{cccccc}
        \toprule
        Method & Architecture & Parameters (M) & \textit{Clean} & \textit{Auto} \\ \midrule
        \citet{rice2020overfitadv} & ResNet-18 & $11.2$ & $88.7$ & $67.7$ \\
        \ours (random) & ResNet-18 & $11.2$ & $\boldsymbol{89.9}$ & $74.0$ 
        \\ 
        \ours (adaptive) & ResNet-18 & $11.2$ & $89.8$ & $\boldsymbol{74.4}$
        \\ \midrule
        \citet{madry2017towards} & ResNet-50 & $23.5$ & $90.8$ & $69.2$ \\
        \citet{gowal2020uncoveringadv} & \wrn{70-16} & $266.8$ & $90.9$  & $74.5$
        \\ \midrule
        \citet{wu2020adversarial} & \wrn{34-10} & $46.2$ & $88.5$ & $73.7$ \\
        \ours (random) & \wrn{34-10} & $46.2$ & $\boldsymbol{90.8}$ & $77.1$ \\
        \ours (adaptive) & \wrn{34-10} & $46.2$ & $\boldsymbol{90.8}$ & $\boldsymbol{77.8}$
        \\ \bottomrule
        \end{tabular}}
     \end{subtable}
\end{table}

\section{Bringing it all together: Using synthetic samples in robust training} \label{app: sota}
In this section, we first analyze the synthetic images and later provide additional analysis based on them in robust training. We will primarily use synthetic images sampled from the \ddpm and \styleC model. 

\textbf{Analyzing synthetic images.} Since \ddpm is an unconditional model, it only generates unlabelled synthetic images. We label these images using a LaNet network on the CIFAR-10 dataset. We visualize the frequency of different classes in Figure~\ref{fig: class_hist}. It shows that the synthetic images are almost uniformly distributed across classes. On the ImageNet dataset, we use the classifier conditioned sampling to generate labeled images from the improved DDPM model~\citep{nichol2021improvedDdpm}. We also provide samples images from different generative models in Figure~\ref{fig: sample_images} and \ref{fig: sample_images_im}.

\textbf{Hyperparameter search for $\gamma$.} Before using \ddpm images in robust training, we perform a hyperparameter search for $\gamma$. We consider ten different values from 0 to 1 and train a ResNet-18 network using \ours at each of them. We measure both the clean accuracy and robust accuracy of each network (Figure~\ref{fig: gamma_search}). We find that $\gamma=0.4$ achieves the highest clean and robust accuracy. Thus we used $\gamma=0.4$ in our experiments. \vikash{Take a pass at content to reflect the new structure of the paper.}

\begin{figure*}[!htb]
    \centering
    \begin{minipage}{0.42\linewidth}
		\centering
	    \includegraphics[width=\linewidth]{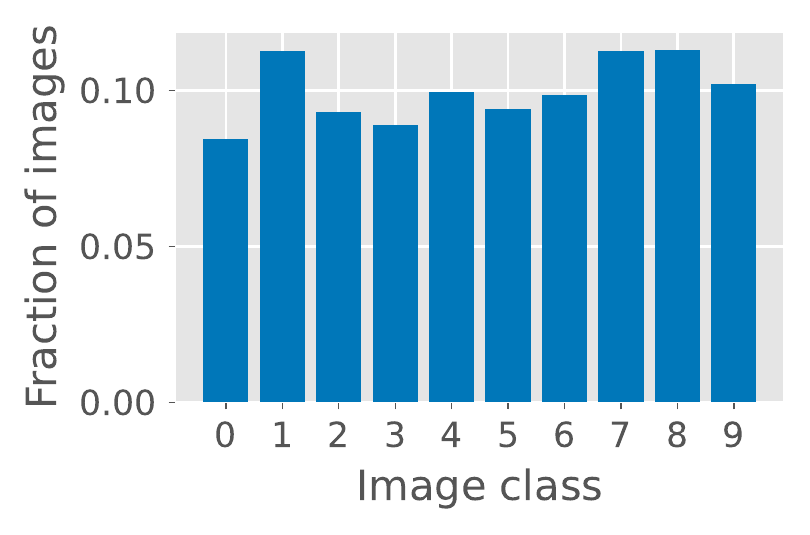}
        \caption{Histogram of class labels for synthetic images from \ddpm model.}
        \label{fig: class_hist}
	\end{minipage}
	\hspace{40pt}
	\begin{minipage}{0.42\linewidth}
		\centering
	     \includegraphics[width=\linewidth]{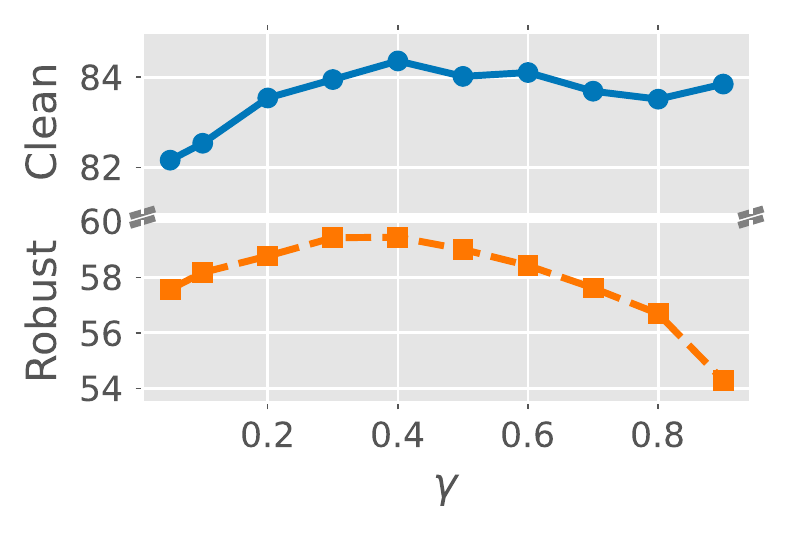}
        \caption{Hyperparameter search for $\gamma$ in \ours.}
        \label{fig: gamma_search}
	\end{minipage}
\end{figure*}

\begin{table}[]
    \centering
    \newbluetext{
    \caption{Hyperparmeter search for $\gamma$ using ResNet-18 on CIFAR-100 dataset.}
    \begin{tabular}{cccccc} \toprule
        $\gamma$ & $0.15$ & $0.40$ & $0.50$ & $0.75$ & $0.90$ \\ \midrule
        \textit{Clean Accuracy} & $60.6$ & $\bf{64.7}$ & $64.5$ & $62.2$ & $59.8$ \\
        \textit{Robust Accuracy} & $27.6$ & $\bf{27.7}$ & $26.4$ & $23.7$ & $21.9$ \\ \bottomrule 
    \end{tabular}
    }
\end{table}

\subsection{Sample complexity of adversarial training} \label{app: sample_complexity}
Given the ability to sample an unlimited amount of synthetic images from a proxy distribution, now we investigate the performance of adversarial training with an increasing number of training samples. We train the network only on synthetic images and measure its performance on another held-out set of synthetic images. We also measure how much the clean and robust accuracy transfers on the CIFAR-10 test set. 

\noindent \textbf{Setup.}  We primarily work with the CIFAR-10 dataset and its two-class subset, i.e., an easier problem of binary classification between class-1 (\textit{automobile}) and class-9 (\textit{truck}). We refer to the latter as CIFAR-2. We robustly train a ResNet-18 network on $2$K to $10$M synthetic images from the \styleC model, as in both $10$-class and $2$-class setup. We opt for \styleC over \ddpm model as sampling images from the former is much faster, thus we were able to generate up to $10$M synthetic from it. Note that the cost of adversarial training increases almost linearly with the number of attack steps and training images. Thus to achieve manageable computational cost when training on millions of images, we opt for using only a 4-step PGD attack (PGD-4) in both training and evaluation. Since robustness achieved with this considerably weak attack may not hold against a strong attack, such as AutoAttack, we opt for evaluating with the PGD-4 attack itself. We also perform natural training, i.e., training on unmodified images in some experiments. We test each network on a fixed set of $100$K images from the \styleC and $10$K images from the CIFAR-10 test set.

\noindent \textbf{Accuracy vs robustness trade-off.} We compare the clean accuracy achieved with both natural and adversarial training in Figure~\ref{fig: acc_vs_robust}. Indeed with a very small number of samples, clean accuracy in adversarial training is traded to achieve robustness. This is evident from the gap between the clean accuracy of natural and adversarial training. However, with the increasing number of training samples, this gap keeps decreasing for both CIFAR-2 and CIFAR-10 datasets. Most interestingly, this trade-off almost vanishes when we use a sufficiently high number of training samples for the CIFAR-2 classification. 

\noindent \textbf{On sample complexity of adversarial training.} We report both clean and robust accuracy with adversarial training in Figure~\ref{fig: generalization}. We find that both clean and robust accuracy continues to improve with the number of training samples. We also observe non-trivial generalization to test images from the CIFAR-10 dataset, which also improves with the number of training samples. Both of these results suggest that even with a small capacity network, such as ResNet-18, adversarial robustness can continue to benefit from an increase in the number of training samples. 

\begin{figure*}[!htb]
	\centering
	\begin{minipage}{0.33\linewidth}
		\centering
		\begin{subfigure}[b]{\linewidth}
            \centering
            \includegraphics[width=0.9\linewidth]{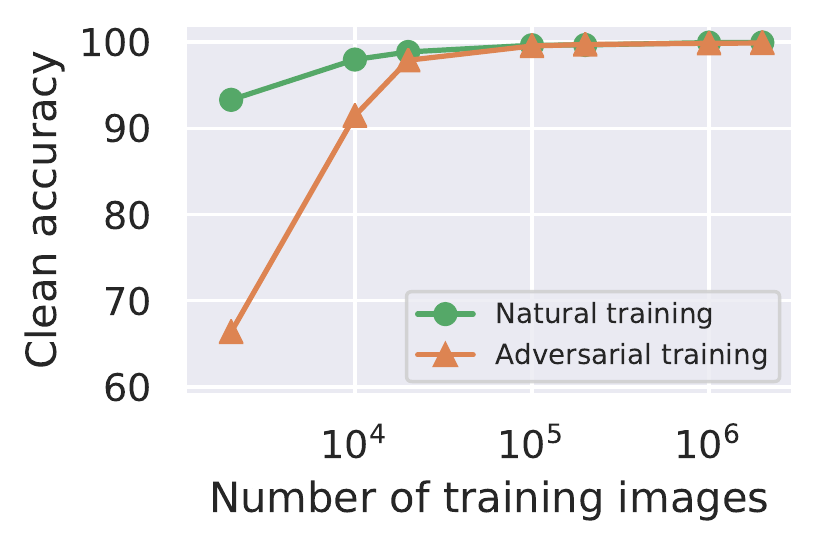}
            \caption{CIFAR-2}
            \vspace{10pt}
        \end{subfigure}
        \begin{subfigure}[b]{\linewidth}
            \centering
            \includegraphics[width=0.9\linewidth]{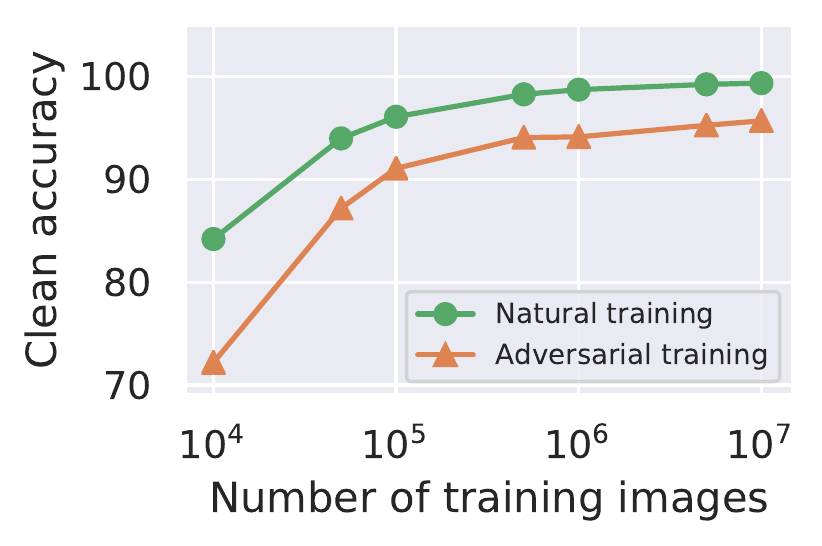}
            \caption{CIFAR-10   }
        \end{subfigure}
        \caption{\textbf{Reduction in accuracy vs robustness trade-off.}  Accuracy vs robustness trade-off when training on an increasing amount of synthetic images from the \styleC model. It shows that the drop in clean accuracy with adversarial training decreases with increase in training samples.}
        \label{fig: acc_vs_robust}
	\end{minipage}
	\hfill
	\begin{minipage}{0.64\linewidth}
		\centering
		\begin{subfigure}[b]{\linewidth}
            \centering
            \includegraphics[width=0.98\linewidth]{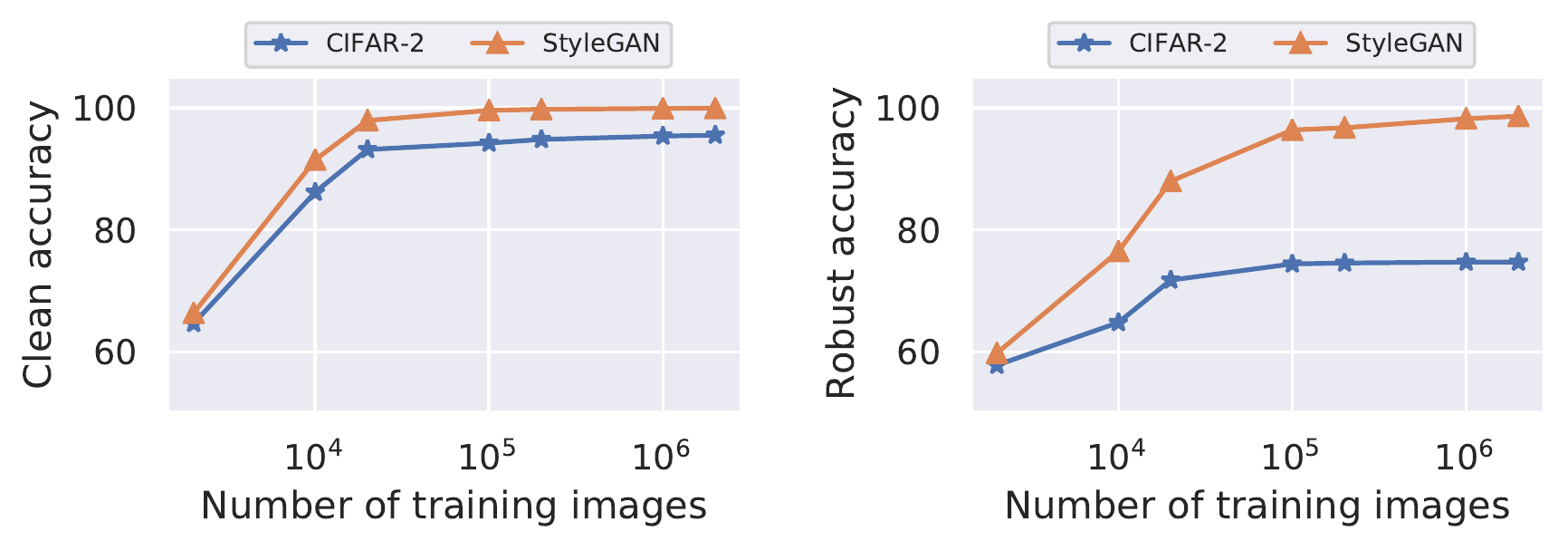}
            \caption{CIFAR-2}
            \vspace{10pt}
        \end{subfigure}
        
        \begin{subfigure}[b]{\linewidth}
            \centering
            \includegraphics[width=0.98\linewidth]{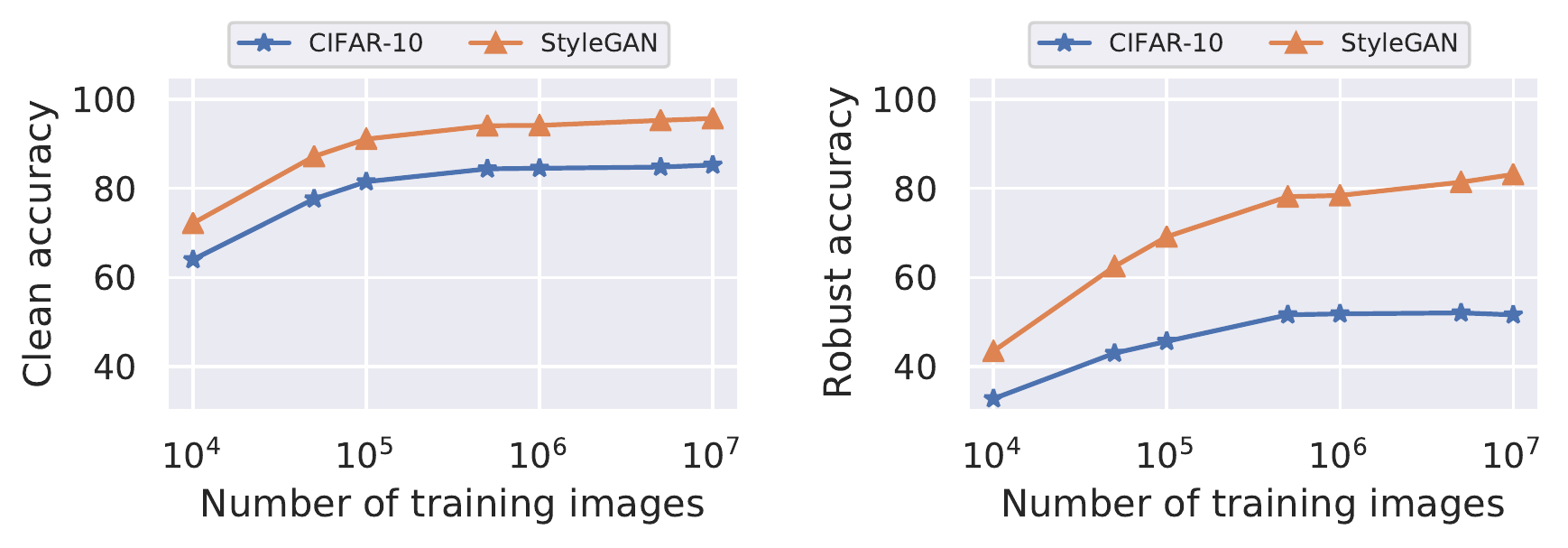}
            \caption{CIFAR-10}
        \end{subfigure}
        \caption{\textbf{Sample complexity of adversarial training.} Clean and robust accuracy on the test set of synthetic samples when trained on an increasing number of synthetic samples from the \styleC model. It shows that performance of adversarial training continues to benefit from increase in number of training samples. We also measure generalization to the CIFAR-10 dataset, which also improves with number of training samples.
        }
        \label{fig: generalization}
	\end{minipage}
\end{figure*}


\subsection{Exploring effect of image quality of synthetic samples}
We explore different classifiers to label the unlabeled synthetic data generated from the \ddpm model. In particular, we use BiT~\citep{Kolesnikov2020BiT}, SplitNet~\citep{zhao2020SplitNet}, and LaNet~\citep{wang2019lanet} where they achieve $98.5$\%, $98.7$\%, and $99.0$\% clean accuracy, respectively, on the CIFAR-10 dataset. We notice that labels generated from different classifiers achieve slightly different downstream performance when used with adversarial training in the proposed approach. 
We find that only up to $10$\% of synthetic images are labeled differently by these networks, which causes these differences. On manual inspection, we find that some of these images are of poor quality, i.e., images that aren't photorealistic or wrongly labeled and remain hard to classify, even for a human labeler. Since filtering millions of images with a human in the loop is extremely costly, we use two deep neural networks, namely LaNet~\citep{wang2019lanet} and SplitNet~\citep{zhao2020SplitNet}, to solve this task. 
We avoid using labels from BiT as it requires transfer learning from ImageNet~\citep{deng2009imagenet} dataset, whereas our goal is to avoid any dependency on extra real-world data. We discard an image when the predicted class of both networks doesn't match and it is classified with less than 90\% confidence by both networks. While the former step flags images which are potentially hard to classify, the latter step ensures that we do not discard images where at least one network is highly confident in its prediction.  We also try the $50$\% and $75$\% confidence threshold but find that $90$\% gives the best downstream results. In particular, using the filtering step improves robust accuracy by another $0.1$\%.

\begin{figure}[!b]
    \centering
    \begin{subfigure}[b]{0.49\linewidth}
        \centering
        \includegraphics[width=0.99\linewidth]{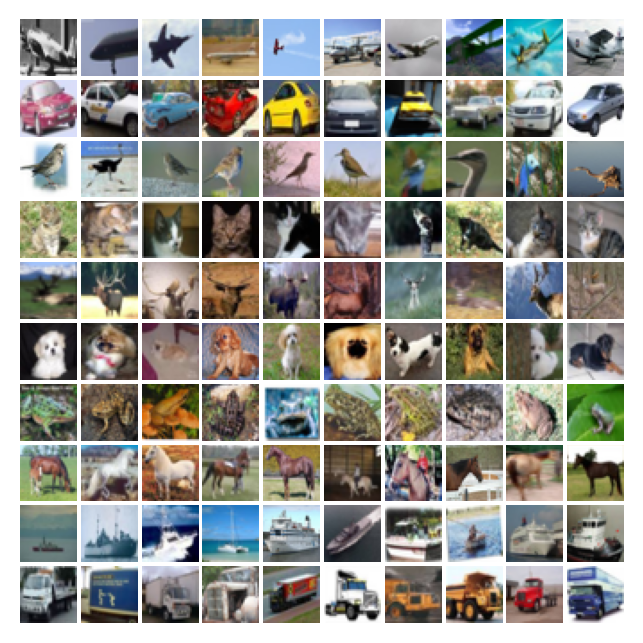}
        \caption{CIFAR-10}
    \end{subfigure}
    \begin{subfigure}[b]{0.49\linewidth}
        \centering
        \includegraphics[width=0.99\linewidth]{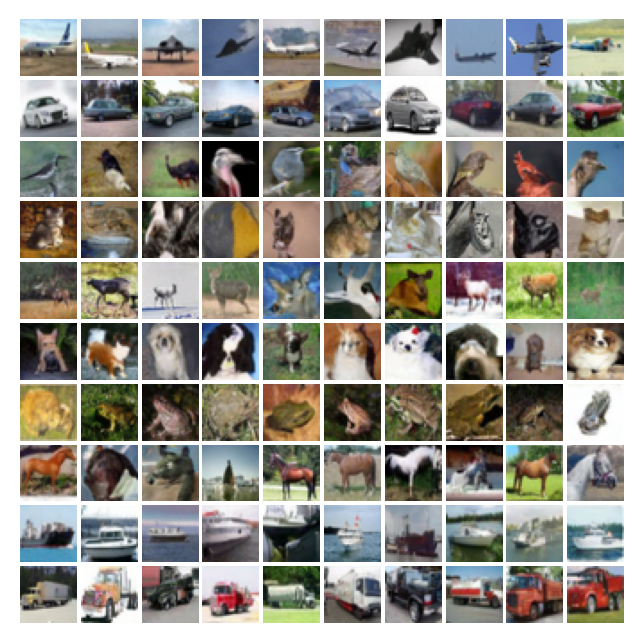}
        \caption{\ddpm}
    \end{subfigure}
    
    \begin{subfigure}[b]{0.49\linewidth}
        \centering
        \includegraphics[width=0.99\linewidth]{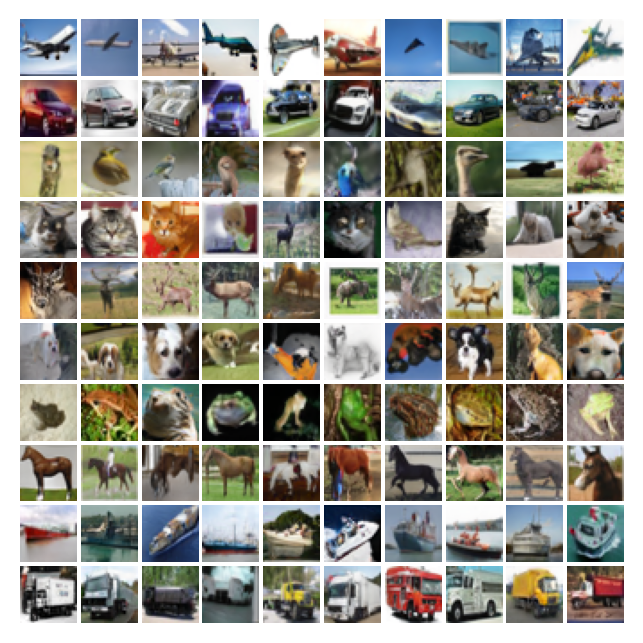}
        \caption{\styleC}
    \end{subfigure}
    \begin{subfigure}[b]{0.49\linewidth}
        \centering
        \includegraphics[width=0.99\linewidth]{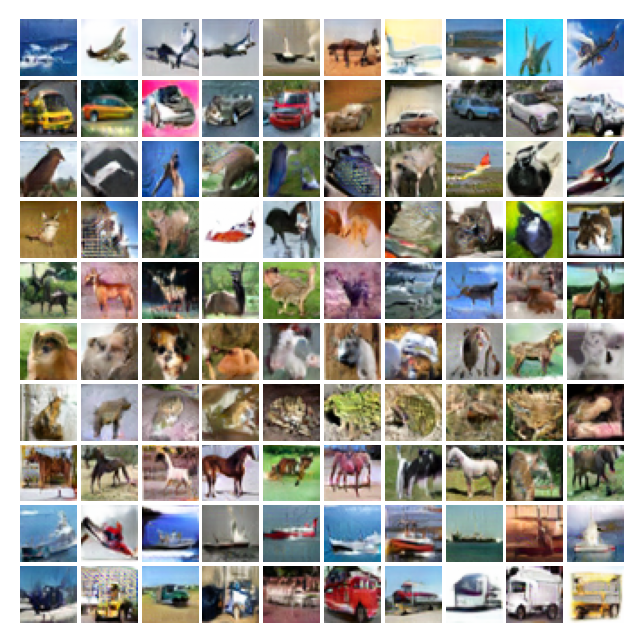}
        \caption{WGAN-ALP}
    \end{subfigure}
    
\caption{\textbf{Visualizing images from different sets.} Randomly selected images form the CIFAR-10 dataset and synthetic images from different generative models. Rows in each figure correspond to following classes: Airplane, automobile, bird, cat, deer, dog, frog, horse, ship, and truck. }
\end{figure}
\begin{figure}[!ht]\ContinuedFloat
    \centering
    \begin{subfigure}[b]{0.49\linewidth}
        \centering
        \includegraphics[width=0.99\linewidth]{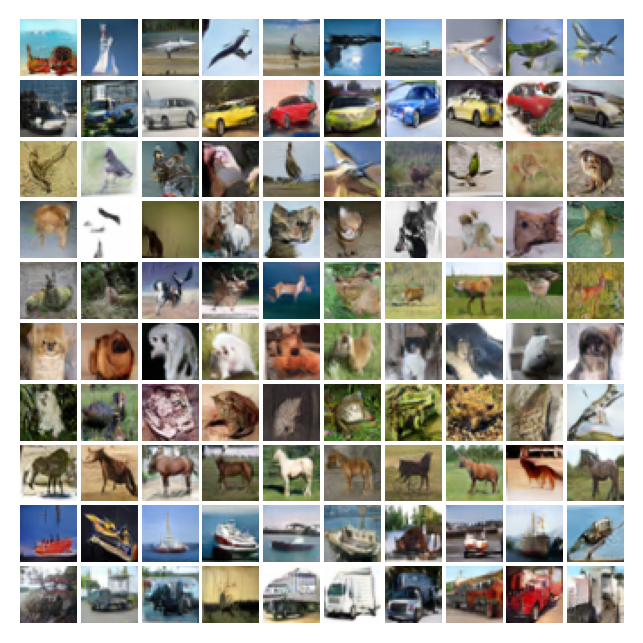}
        \caption{E2GAN}
    \end{subfigure}
    \begin{subfigure}[b]{0.49\linewidth}
        \centering
        \includegraphics[width=0.99\linewidth]{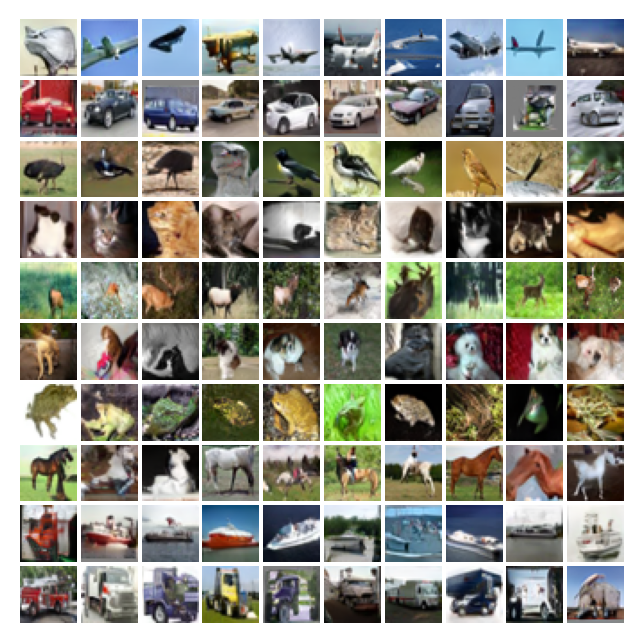}
        \caption{DiffCrBigGAN}
    \end{subfigure}
    
    \begin{subfigure}[b]{0.49\linewidth}
        \centering
        \includegraphics[width=0.99\linewidth]{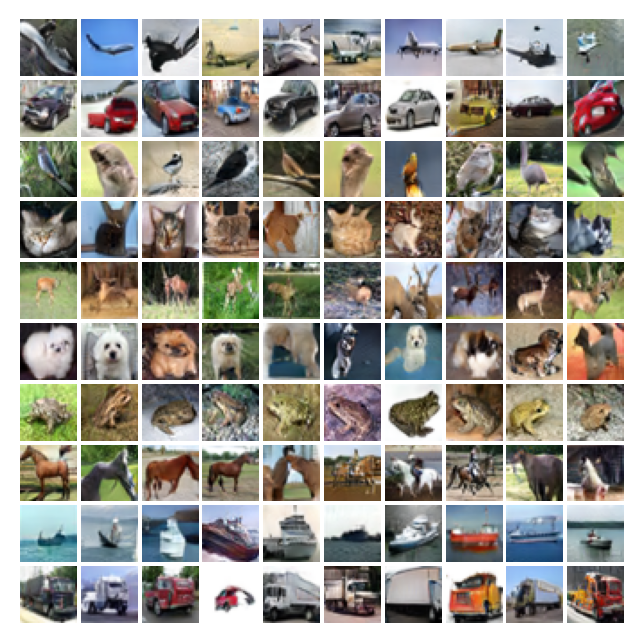}
        \caption{NDA}
    \end{subfigure}
    \begin{subfigure}[b]{0.49\linewidth}
        \centering
        \includegraphics[width=0.99\linewidth]{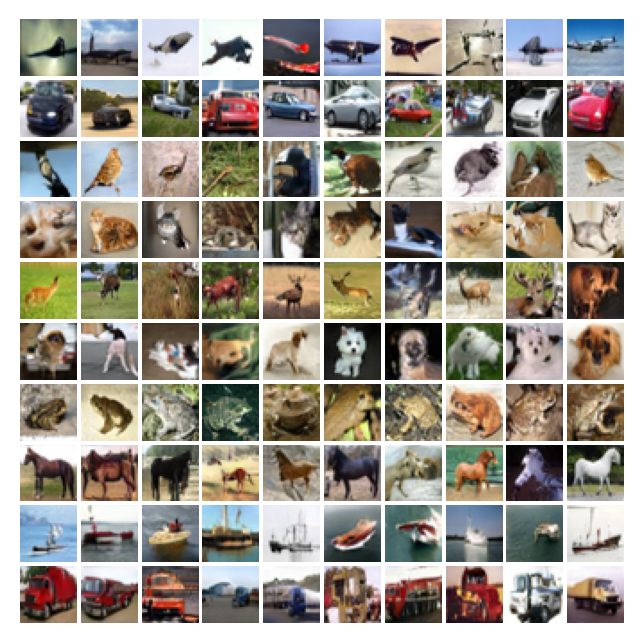}
        \caption{DiffBigGAN}
    \end{subfigure}
\caption{\textbf{Visualizing images from different sets.} Randomly selected synthetic images from different generative models. Rows in each figure correspond to following classes: Airplane, automobile, bird, cat, deer, dog, frog, horse, ship, and truck. }
\label{fig: sample_images}
\end{figure}

\begin{figure}
    \centering
    \begin{subfigure}[b]{0.9\linewidth}
        \centering
        \includegraphics[width=0.99\linewidth]{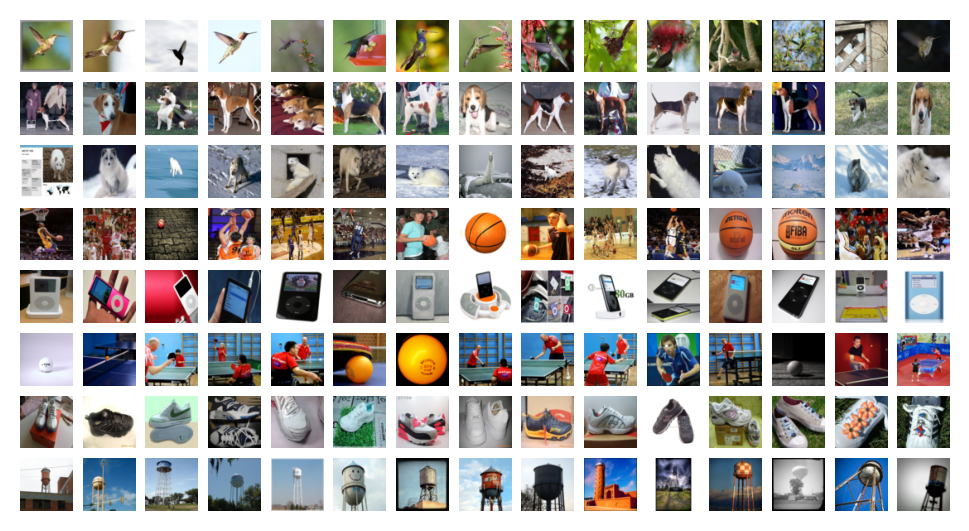}
        \caption{ImageNet~\citep{deng2009imagenet}}
    \end{subfigure}
    
    \begin{subfigure}[b]{0.9\linewidth}
        \centering
        \includegraphics[width=0.99\linewidth]{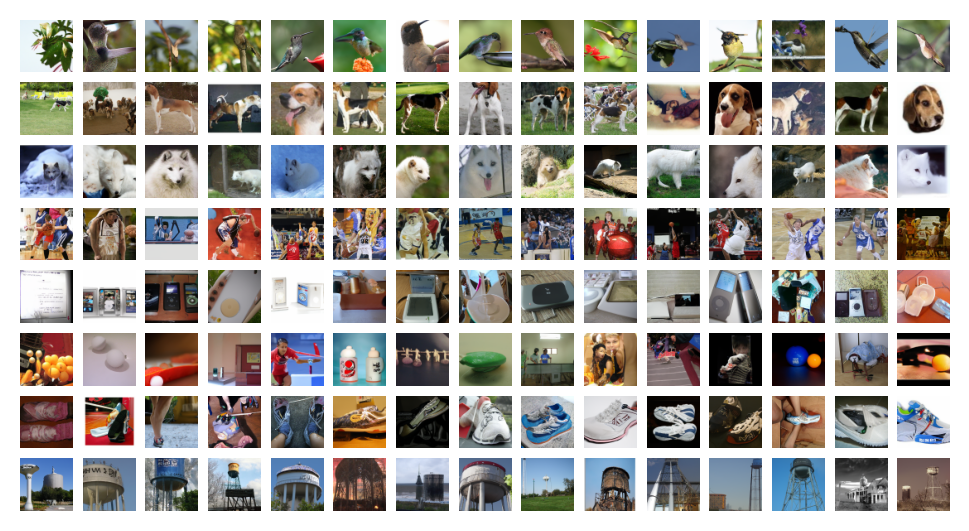}
        \caption{Improved DDPM~\citep{nichol2021improvedDdpm}}
    \end{subfigure}
    
    \begin{subfigure}[b]{0.9\linewidth}
        \centering
        \includegraphics[width=0.99\linewidth]{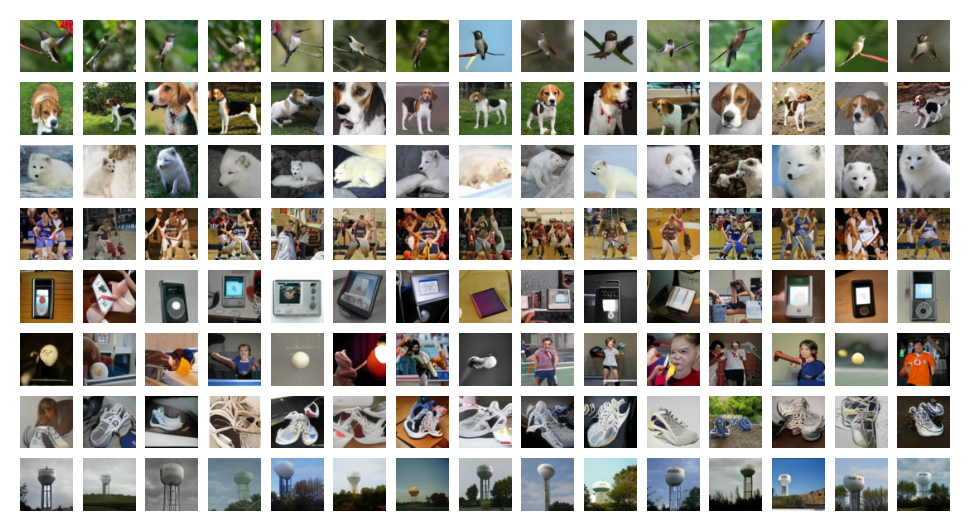}
        \caption{BigGAN-Deep~\citep{brock2018bigGandeep}}
    \end{subfigure}
    \caption{ImageNet ($64\times64$) samples along with synthetic images from two different generative models. Rows correspond to the following classes: hummingbird, english foxhound, arctic fox, basketball, iPod, ping-pong ball, running shoe, and water tower.}
    \label{fig: sample_images_im}
\end{figure}

\begin{figure}[!htb]
    \centering
    \begin{subfigure}[b]{0.9\linewidth}
        \centering
        \includegraphics[width=0.95\linewidth]{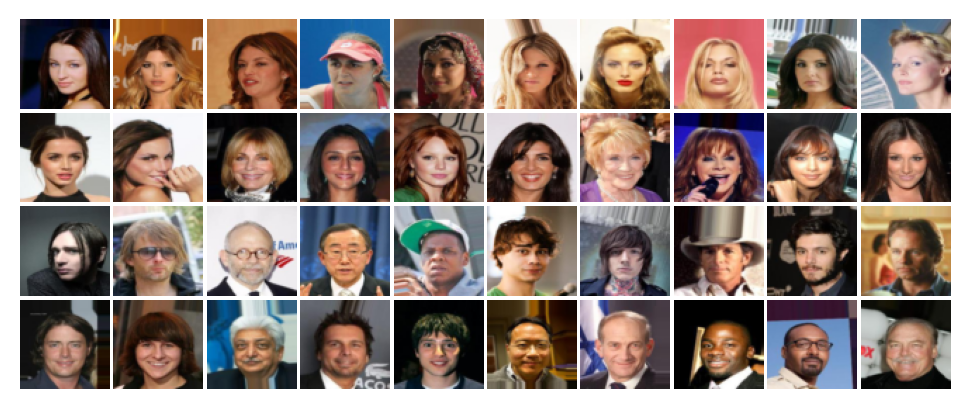}
        \vspace{-8pt}
        \caption{CelebA (Real)}
    \end{subfigure}
    
    \begin{subfigure}[b]{0.9\linewidth}
        \centering
        \includegraphics[width=0.95\linewidth]{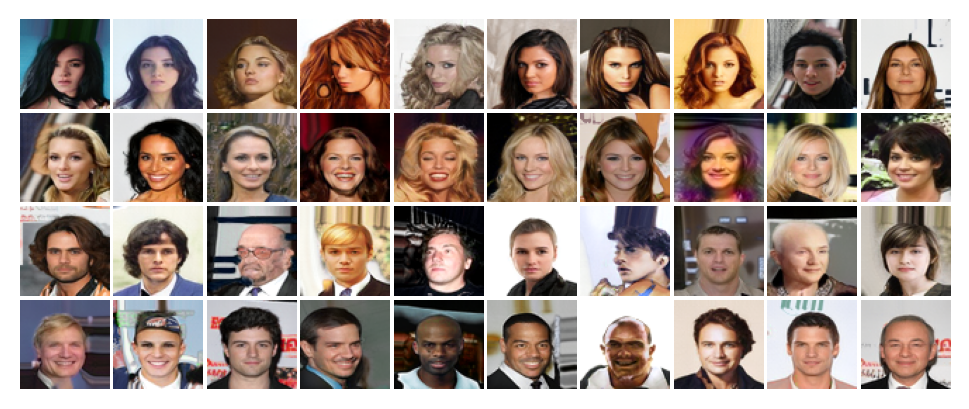}
        \vspace{-8pt}
        \caption{CelebA (Synthetic)}
    \end{subfigure}
    \caption{Real and synthetic images for the CelebA dataset. We consider a four-class classification problem based on the attribute smile and male, i.e., not-smile/not-male, smile/not-male, not-smile/male, smile/male.}
    \vspace{-8pt}
\end{figure}

\begin{figure}[!htb]
    \centering
    \begin{subfigure}[b]{0.9\linewidth}
        \centering
        \includegraphics[width=0.95\linewidth]{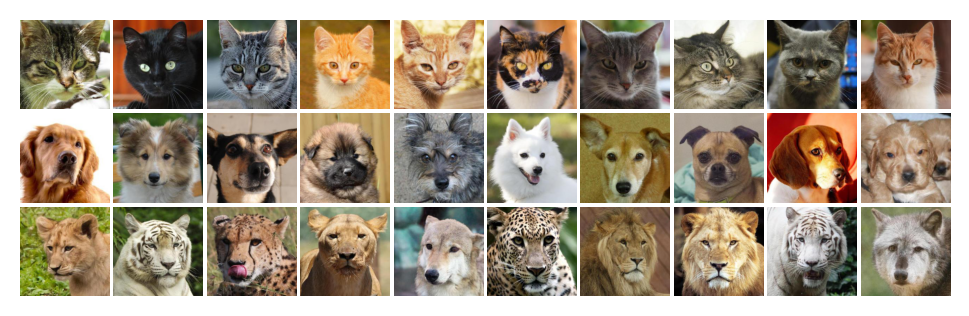}
        \vspace{-8pt}
        \caption{AFHQ (Real)}
    \end{subfigure}
    
    \begin{subfigure}[b]{0.9\linewidth}
        \centering
        \includegraphics[width=0.95\linewidth]{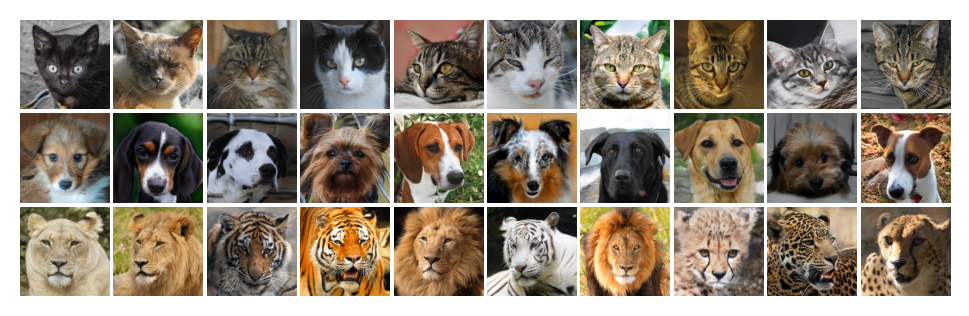}
        \vspace{-8pt}
        \caption{AFHQ (Synthetic)}
    \end{subfigure}
    \caption{Real and synthetic images for the AFHQ dataset. Rows correspond to the following three classes: cat, dog, and wild animals.}
\end{figure}

\end{document}